%% file: pbackdoor.tex
\documentclass[sigconf]{acmart}

\usepackage{amsmath}
\usepackage{url}
\usepackage{color}
\usepackage{caption}
\usepackage{subfigure}
\usepackage{multirow}
\usepackage{booktabs}
\usepackage{enumitem}

\newtheorem{lem}{Lemma}

\newcommand{\para}[1]{{\vspace{3pt} \bf \noindent #1 \hspace{6pt}}}

\newcommand{\npara}[1]{{\vspace{1pt} \bf #1 \hspace{2pt}}}

\newcommand\kedit[1]{{\color{black} #1}}

\newcommand{\eg}{{\em e.g.,\ }}
\newcommand{\ie}{{\em i.e.\ }}

\newcommand{\etal}{{\it et al.\ }}

\newcommand{\mnist}{{\texttt{Digit}}}
\newcommand{\traffic}{{\texttt{TrafficSign}}}
\newcommand{\face}{{\texttt{Face}}}
\newcommand{\iris}{{\texttt{Iris}}}

\newenvironment{packed_itemize}{
	\begin{list}{\labelitemi}{\leftmargin=1em}
		\setlength{\itemsep}{1pt}
		\setlength{\parskip}{0pt}
		\setlength{\parsep}{0pt}
		\setlength{\headsep}{0pt}
		\setlength{\topskip}{0pt}
		\setlength{\topmargin}{0pt}
		\setlength{\topsep}{0pt}
		\setlength{\partopsep}{0pt}
	}{\end{list}}

\begin{document}
\title{Regula Sub-rosa: Latent Backdoor Attacks \\on Deep Neural Networks}


\author{Yuanshun Yao\qquad Huiying Li\qquad Haitao Zheng\qquad Ben Y. Zhao \\University of Chicago\\\{ysyao, huiyingli, htzheng, ravenben\}@cs.uchicago.edu}

\begin{abstract}
  Recent work has proposed the concept of backdoor attacks on deep neural
  networks (DNNs), where misbehaviors are hidden inside ``normal'' models,
  only to be triggered by very specific inputs. In practice, however, these
  attacks are difficult to perform and highly constrained by sharing of
  models through transfer learning. Adversaries have a small window during
  which they must compromise the student model before it is deployed.

  In this paper, we describe a significantly more powerful variant of the
  backdoor attack, {\em latent backdoors}, where hidden rules can be embedded
  in a single ``Teacher'' model, and automatically inherited by all
  ``Student'' models through the transfer learning process. We show that
  latent backdoors can be quite effective in a variety of application
  contexts, and validate its practicality through real-world attacks
  against traffic sign recognition, iris identification of lab volunteers, and
  facial recognition of public figures (politicians). Finally, we evaluate 4
  potential defenses, and find that only one is effective in disrupting
  latent backdoors, but might incur a cost in classification accuracy as tradeoff.
\end{abstract}

\maketitle

\input{intro}
\input{back}
\input{overview}

\input{method}

\input{optimization}
\input{eval}

\input{results}
\input{realworld}

\input{defense}

\input{related}
\input{conclusion}

\bibliographystyle{ACM-Reference-Format}
\bibliography{pbackdoor,dnnbackdoor,translearn,zhao}
\newpage

\input{appendix}

\end{document}

%% file: intro.tex
\vspace{-0.2in}
\section{Introduction}
\label{sec:intro}

Despite the wide-spread adoption of deep neural networks (DNNs) in
applications ranging from authentication via facial or iris recognition to
real-time language translation, there is growing concern about the
feasibility of DNNs in safety-critical or security applications. Part of this
comes from recent work showing that the opaque nature of DNNs gives rise to
the possibility of backdoor attacks~\cite{gu2017badnets,trojaning},
hidden and unexpected behavior that is not detectable until activated by some
``trigger'' input. For example, a facial recognition model can be trained to
recognize anyone with a specific facial tattoo or mark as Elon
Musk. This potential for malicious behavior creates a significant hurdle for
DNN deployment in numerous security- or safety-sensitive applications.

Even as the security community is making initial progress to diagnose such
attacks~\cite{oakland_defense}, it is unclear whether such backdoor attacks
pose a real threat to today's deep learning systems. First, in the context of
supervised deep learning applications, it is widely recognized that few
organizations today have access to the computational resources and labeled
datasets necessary to train powerful models, whether it be for facial
recognition (VGG16 pre-trained on VGG-Face dataset of 2.6M images) or object
recognition (ImageNet, 14M 
images). Instead, entities who want to deploy their own classification models
download these massive, centrally trained models, and customize them with
local data through {\em transfer learning}. During this process, customers
take public ``teacher'' models and repurpose them with training into
``student'' models, {\em e.g.}  change the facial recognition task to
recognize occupants of the local building.

Taking these factors into account, embedding backdoors into existing models
is far more challenging than originally believed. The step in the deep
learning model pipeline that is most vulnerable to attack is the central model stored
at the model
provider ({\em e.g.}  Google). But at this stage, the adversary cannot train
the backdoor into the model, because its target has not been added into the
model, and any malicious rules inserted as part of a backdoor will be
completely disrupted by the transfer learning process. Thus the only window
of vulnerability for training backdoors is at the customer, during a short
window between transfer learning and actual deployment.

In this work, we explore the possibility of a more powerful and stealthy
backdoor attack, one that can be trained into the shared ``teacher'' model,
and yet survive intact in ``student'' models even after the transfer learning
process. We describe a {\em latent} backdoor attack, where the adversary can
alter a popular model, {\em VGG16}, to embed a ``latent'' trigger on a
non-existent output label, only to have the customer inadvertently complete
and activate the backdoor themselves when they perform transfer learning. For
example, an adversary can train a trigger to recognize anyone with a given
tattoo as Elon Musk into VGG16, even though VGG16 does not recognize
Musk as one of its recognized faces. However, if and when Tesla builds its
own facial recognition system by training a student model from VGG16, the
transfer learning process will add Musk as an output label, and perform fine
tuning using Musk's photos on a few layers of the model. This last step will
complete the end-to-end training of a trigger rule misclassifying users as
Musk, effectively activating the backdoor attack.

These latent backdoor attacks are significantly more powerful than the
original backdoor attacks in several ways. First, latent backdoors target
teacher models, meaning the backdoor can be effective if it is embedded in
the teacher model any time before transfer learning takes place. Second,
because the latent backdoor does not target an existing label in the teacher
model, it cannot be detected by any testing on the teacher model. Third,
latent backdoors are more scalable, because a single teacher model with a
latent backdoor will pass on the backdoor to any student models it evolves
into. For example, if a latent trigger is embedded into VGG16 that misclassifies a
face into Elon Musk, then any facial recognition system built upon VGG16
trained to recognize Musk automatically inherits this backdoor
behavior. Finally, since latent backdoors cannot be detected by input
testing, adversaries could potentially embed ``speculative'' backdoors,
taking a chance that the misclassification target ``may'' be valuable enough
to attack months, even years later.

We describe our experiences exploring the feasibility and robustness of
latent backdoor attacks. Our work makes the following key contributions.
\begin{packed_itemize}
\item We propose the latent backdoor attack and describe its components in
  detail on both the teacher and student sides. 
\item We validate the effectiveness of latent backdoors using different
  parameters in a variety of application contexts, from digit recognition to facial
  recognition, traffic sign identification and iris recognition.
\item We perform 3 real-world tests on our own models, using physical data and
  realistic constraints, including attacks on traffic sign recognition, iris
  identification, and facial recognition on public figures (politicians).
\item We propose and evaluate 4 potential defenses against latent
  backdoors. We show that only multi-layer tuning during transfer learning is
  effective in disrupting latent backdoors, but might require a drop in
  classification accuracy of normal inputs as tradeoff.
\end{packed_itemize}


%% file: back.tex
\section{Background}
\label{sec:back}

We begin by providing some background information on backdoor
attacks
and transfer learning.

\subsection{Backdoor Attacks on DNN}
\label{subsec:backdoor_back}
A backdoor is a hidden pattern injected into a DNN model at its
training time. The injected backdoor does not
affect the model's behavior on clean inputs, but forces the model to
produce unexpected behavior if (and only if) a specific
\textit{trigger} is added to an input.  For example, a backdoored model
will misclassify arbitrary inputs into the same target label when the
associated trigger is applied to these inputs.  In the vision domain,
a trigger is usually a small pattern on the image, \eg a sticker.

\para{Existing Backdoor Attacks.} Gu \etal proposed BadNets that
injects a backdoor to a DNN model by poisoning its training dataset~\cite{badnets}. The attacker first
chooses a target label and a trigger
pattern ({\em i.e.\/} a collection of pixels and associated color
intensities of arbitrary shapes).  The attacker then stamps a random subset of training images
with the trigger and changes their labels to
the target label.  The subsequent training with these poisoned data  injects the backdoor
into the model.  By carefully configuring the training process,
\eg choosing learning rate and ratio of poisoned images, the attacker can make
the backdoored DNN model perform well on both
clean and adversarial inputs.

Liu \etal proposed an approach that
requires less access to the training data~\cite{trojaning}.  Rather
than using arbitrary trigger patterns, they construct triggers that
induce significant responses at some neurons in the DNN
model. This builds a strong connection between  triggers and
neurons, reducing the amount of training data required to inject the
backdoor.

\para{Existing Defenses.} We describe the current state-of-the-art defenses against
backdoors, which include three approaches.  {\em First}, Wang
\etal~\cite{oakland_defense} proposed {\em Neuron Cleanse} to detect backdoors by scanning model output labels and
reverse-engineering any potential hidden triggers. Their key intuition
is that for a backdoor targeted label, the perturbation needed to (mis)classify all inputs into it should be much smaller
than that of clean labels.  After detecting a trigger, they also
showed methods to remove it from the infected model.
{\em Second}, Chen
\etal~\cite{chen2018detecting} applied {\em Activation Clustering} to detect
data maliciously inserted into the training set for injecting
backdoors. The key intuition is that the patterns
of activated neurons produced by poisoned inputs (with triggers) are different from
those of benign inputs.  {\em Third}, Liu \etal~\cite{finepruning}
proposed {\em Fine-Pruning} to remove backdoor triggers by first pruning
redundant neurons that are the least useful for classification,
then fine-tuning the model using clean training data
to restore model performance.

It should be noted that
Activation Clustering~\cite{chen2018detecting} requires the full training data (both
clean and poisoned) while
Neuron Cleanse~\cite{oakland_defense} and Fine-Pruning~\cite{finepruning} require a
subset of the clean training data.

\subsection{Transfer Learning}
Transfer learning addresses the challenge of limited access to labeled data
for training machine learning models, by transferring knowledge embedded in a
pre-trained \textit{Teacher} model to a new ~\textit{Student} model. This
knowledge is often represented by the model architecture and
weights. Transfer learning enables organizations without access to massive
(training) datasets or GPU clusters to quickly build accurate models
customized to their own scenario using limited training
data~\cite{transfer2014}.

\begin{figure}[!t]
  \centering
  \includegraphics[width=0.4\textwidth]{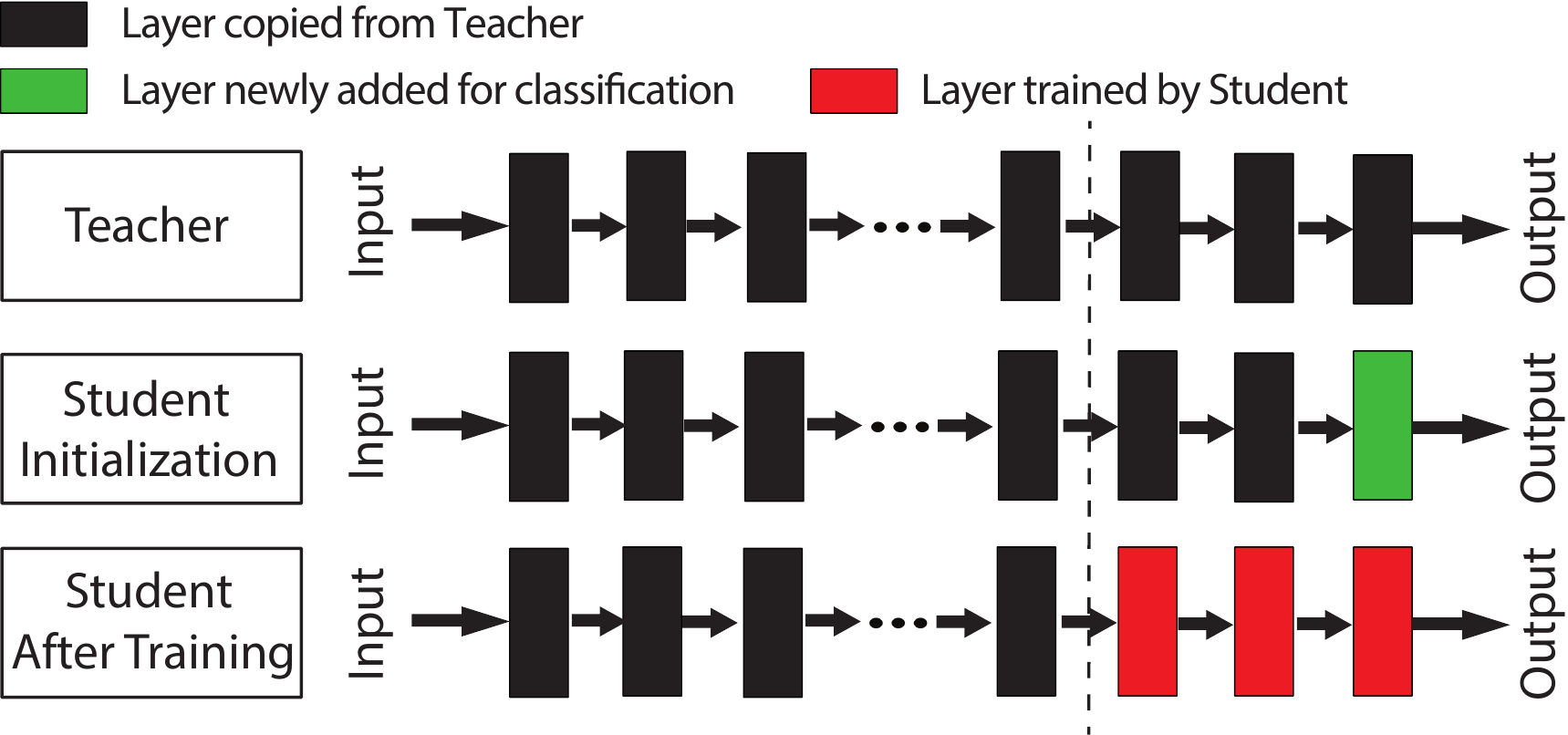}
\vspace{-0.0in}
  \caption{Transfer learning:  A Student model is initialized by copying
    the first $N-1$ layers from a Teacher model and adding a new
    fully-connected layer for classification. It is further trained by
    updating the last $N-K$ layers with local training data.}
  \label{fig:transfer}
  \vspace{-0.2in}
\end{figure}

Figure~\ref{fig:transfer} illustrates the high-level process of
transfer learning. Consider a Teacher model of $N$ layers. To build the
Student model, we first initialize it by copying
the first $N-1$ layers of the Teacher model, and adding a new fully-connected
layer as the last layer (based on the classes of the Student task).
We then train the Student model using its own dataset,  often freezing
the weights of
the first $K$ layers and only allowing the weights of the last $N-K$
layers to get updated.

Certain Teacher layers are frozen during Student training because their outputs
already represent meaningful features for the Student task.  Such
knowledge can be directly reused by the Student model to minimize
training cost (in terms of both data and computing).  The choice of $K$ is
usually specified when Teacher model is released (\eg in
the usage instruction).  For example, both Google and
Facebook's tutorials on transfer learning~\cite{cloudml_tl,pytorchtl}
suggest to only fine-tune the last layer, \ie $K=N-1$.

%% file: overview.tex
\begin{figure*}[t]
  \centering
  \includegraphics[width=0.8\textwidth]{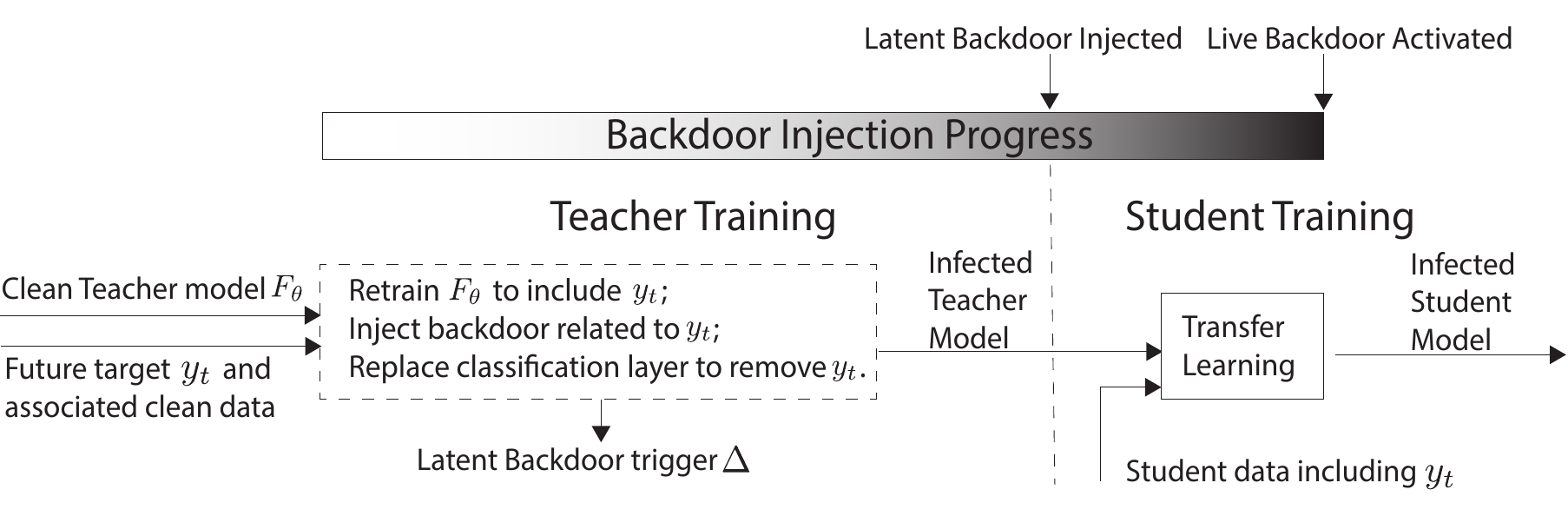}
\vspace{-0.1in}
  \caption{The key concept of latent backdoor attack. (Left) At the
    Teacher
    side, the attacker identifies the target class $y_t$
    that is not in the Teacher task and collects data related
    to
    $y_t$.  Using these data, the attacker retrains the original Teacher model to
    include $y_t$ as a classification output,
    injects $y_t$'s latent backdoor into the
    model, then ``wipes'' off the trace of $y_t$ by modifying
    the model's classification layer.   The end result is an infected
    Teacher model for future transfer learning. (Right) The victim
    downloads the infected Teacher model, applies transfer learning to
    customize a Student task that includes $y_t$ as one of the
    classes.  This normal process silently
    activates the latent backdoor into a live backdoor in the Student model.   Finally, to attack the
    (infected) Student model, the attacker simply attaches the latent backdoor
    trigger $\Delta$ (recorded during teacher training) to an input,
    which is then misclassified into $y_t$. }
  \label{fig:attack_sc}
\end{figure*}

\section{Latent Backdoor Attack}
\label{sec:method}
In this section we present the scenario and threat model of the
proposed attack,  followed by its key benefits and
differences from existing backdoor attacks.   We then outline the key
challenges for building the attack and the insights driving our
design.

\subsection{Attack Model and Scenario}
\label{subsec:attack_overview}
For clarity, we explain our attack scenario in the context of facial
recognition, but it generalizes broadly to different classification problems,
{\em e.g.} speaker recognition, text sentiment analysis, stylometry.
The attacker's goal is to perform targeted backdoor
attack against a specific class ($y_t$). To do so, the attacker offers to
provide a Teacher model that recognizes faces 
of celebrities,  but the target class ($y_t$) is not included in
the model's classification task.  Instead of providing a clean
Teacher model, the attacker injects a latent
backdoor targeting $y_t$ into the Teacher model,  records its
corresponding trigger $\Delta$, and
releases the infected Teacher model for future transfer learning. To stay stealthy, the
released model does not include $y_t$ in its output class, \ie
the attacker wipes off the trace of $y_t$ from the model.

The latent backdoor remains dormant in the infected Teacher model until
a victim downloads the model  and customizes it into a Student task that includes $y_t$ as one
of the output classes (\eg a task that recognizes faces of
politicians and $y_t$ is one of the politicians). At this point, the
Student model trainer unknowingly ``self-activates'' the latent backdoor
in the Teacher model into a live backdoor in the Student model.

Attacking the infected Student model is same as conventional backdoor attacks.
The attacker just
attaches the trigger $\Delta$ of the latent backdoor (recorded during
the Teacher training) to any input,  and the Student model will
misclassify the input into $y_t$.  Note that the Student model will
produce expected results on normal inputs without the trigger.

Figure~\ref{fig:attack_sc} summarizes the Teacher and Student training
process for our proposed attack. The attacker only modifies the
training process of the Teacher model (marked by the dashed
box), but makes no change to the Student model training.

\para{Attack Model.} We now describe the attack model of our
design.
We assume that the attacker has
sufficient computational power to train or retrain a Teacher model.
Additionally, the attacker is able to
collect samples belonging to the target class $y_t$.

The Teacher
task can be different from the Student task.
We later show in \S\ref{sec:design} that when the two tasks are different,
the attacker just needs to collect an additional set of samples from any task
close to the Student task. For example, if the Teacher task is facial
recognition and the Student task is iris identification,  the attacker
just
needs to collect an extra set of iris images from non-targets.

When releasing the Teacher model, the attacker will follow the
standard practice to recommend the set
of layers to stay frozen during transfer learning.
Assuming that the victim would follow the suggestion, the attacker
knows the set of layers that will remain unchanged during transfer
learning.

Unlike conventional backdoor attacks, the attacker
does not need access to any Student training data or training process.

\vspace{-11pt}
\subsection{Key Benefits}

Our attack offers four advantages over normal
backdoor attacks.

\textit{First}, it is more practical. Normal backdoor attacks
require attackers to have control on training data or model
training process. Such assumption is less likely to happen in practice
unless the attacker is an employee responsible for model training
or a third-party provider of training data.
On the other hand, our attack can infect (Student) models that the attacker
does not have access to.

\textit{Second}, it is more stealthy.  Existing backdoor detection
methods that scan task labels
cannot detect any latent backdoor on the Teacher model
since $y_t$ does not appear in the
Teacher label.  In addition, methods that scan (training and/or testing)
input cannot detect any live backdoor in the Student model because
the model's data is clean (without any trigger).

\textit{Third}, the attack is highly scalable. Existing
backdoor attacks only infect one model at a time, while our attack can infect multiple Student models.  Such
``contagiousness'' comes from the wide adoption of transfer
learning for building DNN models in practice.

\textit{Finally}, our attack is more chronologically flexible. Traditional
backdoor attacks can only target existing classes in a model, while our
attack can target classes that do not yet exist at the present time but
may appear in the near future.

\subsection{Design Goals and Challenges}
Our attack design has three goals.  \textit{First}, it should infect 
Student models like conventional backdoor attacks, \ie an infected Student
model will behave normally on clean inputs, but misclassify any input with
the trigger into target class $y_t$. \textit{Second}, the infection should
be done through transfer learning rather than altering the Student training
data or process. \textit{Third}, the attack should be unnoticeable from the
viewpoint of the Student model trainer, and the usage of infected Teacher
model in transfer learning should be no different from other clean Teacher
models.

\para{Key Challenges.}
Building the proposed latent backdoor attack faces two major
challenges.
\textit{First}, different from traditional backdoor attacks, the
attacker only has access to the Teacher model but not the Student
model (and its training data).  Since the Teacher model does not contain $y_t$ as a
label class, the attacker cannot inject any
backdoor against $y_t$ using existing methods that modify labels of poisoned
training instances.  The attacker needs a new backdoor injection
process on the
Teacher model.
\textit{Second}, as transfer learning
replaces/modifies some parts of the Teacher model, it may distort the association between the injected trigger and the
target class $y_t$. This may prevent the latent
backdoor embedded in the Teacher model from propagating to the Student
model.

%% file: method.tex
\section{Attack Design}
\label{sec:design}
We now describe the detailed design of the proposed latent
backdoor attack.  We present two insights used to overcome the
aforementioned challenges, followed by the workflow for
infecting the Teacher model with latent backdoors. Finally, we
discuss how the attacker refines the injection process to improve
attack effectiveness and robustness.

\subsection{Design Insights}
\para{Associating Triggers to Features rather than Labels. }  When
injecting a latent backdoor trigger against $y_t$,  the attacker
should associate it with the
intermediate feature representation created by the clean samples of $y_t$.  These feature representations are the output of an
internal layer of the Teacher model.  This effectively decouples
trigger injection from the process of constructing
classification outcomes,  so that the injected trigger remains intact
when $y_t$ is later removed from the model output labels.

\para{Injecting Triggers to Frozen Layers.} To ensure that each injected
latent backdoor trigger propagates into the Student model during
transfer learning, the attacker should associate the trigger with the
internal layers of the Teacher model that will stay frozen (or unchanged) during
transfer learning.  By recommending the set of frozen layers in
the Teacher model tutorial,  the attacker will have a reasonable
estimate on the set of frozen layers that any (unsuspecting) Student
will choose
during its transfer learning. Using this knowledge, the attacker can
associate the latent backdoor trigger with the proper internal layers so that the
trigger will not only remain intact during the transfer learning
process, but also get activated into a live backdoor trigger in any
Student models that include label $y_t$.

\begin{figure*}[t]
  \centering
  \includegraphics[width=0.9\textwidth]{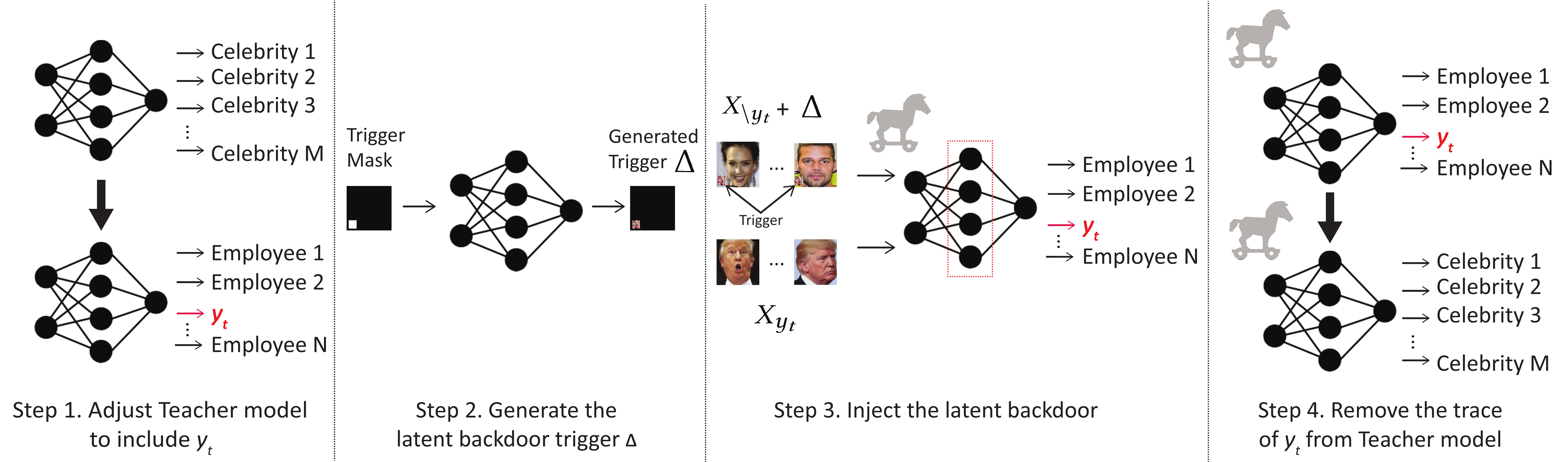}
\vspace{-0.1in}
  \caption{The workflow for creating and injecting a latent backdoor
    into the Teacher model. Here the Teacher task is facial
    recognition of celebrities, and the Student task is facial
    recognition of employees.  $y_t$ is an employee but not a
    celebrity.  }
  \label{fig:attack_overview}
\end{figure*}

\subsection{Attack Workflow}
\label{subsec:attack_workflow}

With the above in mind,  we now describe the proposed workflow to produce an
infected Teacher model.  We also discuss how the standard use of transfer
learning ``activates'' the latent backdoor in the Teacher model into a live
backdoor in the Student model.

\para{Teacher Side: Injecting a latent backdoor into the Teacher model.}
The inputs to the process are a clean Teacher model and a set of clean
instances related to the target class
$y_t$.  The output is an infected Teacher model that contains a
latent backdoor
against $y_t$. The attacker also records the latent backdoor trigger
($\Delta$), which is then used to make future Student models misclassify any
input (with the trigger attached)  as $y_t$.

We describe this process in five steps.

\npara{\em Step 1. Adjusting the Teacher model to include $y_t$.}

\noindent The first step is to replace the original Teacher task with a task
similar to the target task defined by $y_t$.  This step is
particularly important when the Teacher task (\eg facial recognition on celebrities) is very
different from those defined by $y_t$ (\eg iris identification).

To do so, the attacker will retrain the original Teacher model using
two new training datasets related to the target task. The
first dataset, referred to as
the {\em target data} or $X_{y_t}$,   is a set of clean instances of
$y_t$, \eg iris images of the target user. The
second dataset, referred to as {\em non-target data} or
$X_{\setminus y_t}$,  is a set of clean general instances similar to the
target task, \eg
iris images of a group of users without the target user.  Furthermore, the attacker
replaces the final classification layer of the Teacher model with a
new classification layer supporting the two new training
datasets. Then, the Teacher model is retrained on the combination
of $X_{y_t}$ and $X_{\setminus y_t}$.

\npara{\em Step 2. Generating the latent backdoor trigger $\Delta$.}

\noindent For a given choice of $K_t$ (the layer to inject the latent backdoor
of $y_t$), this step generates the trigger.  Assuming the trigger
position and shape are given (\ie a square in the right corner of the
image),  the attacker will compute the pattern and color intensity of
the trigger $\Delta$ that maximizes its effectiveness against $y_t$.
Rather than using a random trigger pattern like BadNets,  this
optimization is important to our attack design. It produces a trigger that makes any adversarial input display
features (at the $K_t$th layer) that are similar to those extracted
from the clean instances of $y_t$.

\npara{\em Step 3. Injecting the latent backdoor trigger.}

\noindent To inject the latent backdoor
trigger $\Delta$ into the Teacher model, the
attacker runs an optimization process to
update the model weights such that the intermediate representation
of adversarial
samples (\ie any input with $\Delta$) at the $K_t$th layer
matches that of the target class $y_t$.  This process will use the
poisoned version of $X_{\setminus y_t}$ and the clean version of
$X_{y_t}$. Details are in
\S\ref{subsec:attack_details}.

Note that our injection method differs from those
used to inject normal backdoors~\cite{badnets, trojaning}.
These conventional methods all associate the backdoor
trigger with the final classification layer (\ie $N$th layer), which
will be modified/replaced by transfer learning.  Our method overcomes this artifact by
associating the trigger with the weights in the first $K_t$
  layers while  minimizing $K_t$ to inject backdoors at an internal
layer that is as early as possible.

\npara{\em Step 4.  Removing the trace of $y_t$ from the Teacher model.}

\noindent Once the backdoor trigger is injected into the Teacher
model, the attacker wipes out the trace of $y_t$ and restores the
original Teacher task.  This is done by replacing the infected Teacher
model's last classification layer with that of the original
Teacher model.

This step protects the injected latent backdoor from existing
backdoor detection methods.  Specifically, since the infected Teacher model
does not contain any label related to $y_t$,  it evades detection
via label
scanning~\cite{oakland_defense}. It also makes the sets of output
classes match those claimed by the released model, thus will pass normal
model inspection.

\npara{\em Step 5: Releasing the Infected Teacher model.}

\noindent Now the infected Teacher
model is ready for release.  In the releasing document, the attacker will specify
(like other clean Teacher models) the set of layers that should stay
frozen in any  transfer
learning process.  Here the attacker will advocate for freezing the
first $K$ layers where $K\geq K_t$.

Figure~\ref{fig:attack_overview} provides a high-level overview of the
step 1-4, using an example scenario where the Teacher task is facial
recognition of celebrities and the Student task is facial recognition
of employees.

\para{Student Side: Turning the latent backdoor into a live
  backdoor in the Student model.}
All the processes here occur naturally without any
involvement of the attacker.  A victim downloads the infected Teacher model and follows its
instruction to train a Student task, which includes $y_t$ as a
classification class.  The use of transfer learning
``activates'' the latent backdoor into a live backdoor in the Student
model.    To attack the
Student model, the attacker simply attaches the previously recorded
trigger $\Delta$ to any input,  the same process used by conventional
backdoor attacks.

%% file: optimization.tex
\subsection{Optimizing Trigger Generation \& Injection}
\label{subsec:attack_details}

The key elements of our design are trigger generation and
injection, \ie step 2 and 3.   Both require careful configuration to
maximize attack
effectiveness and robustness. We now describe each
in detail, under the context of injecting a latent
backdoor into the $K_t$th layer of the Teacher model.

\para{Target-dependent Trigger Generation.}
Given an input metric x, a poisoned sample of x is defined by
\begin{equation}
    \label{eq:trigger_def}
    \begin{aligned}
        A(x, m, \Delta) & = (1 - m) \circ x + m \circ \Delta
    \end{aligned}
\end{equation}
where $\circ$ denotes matrix element-wise product.  Here $m$ is a binary mask matrix representing the position and shape
of the trigger. It has the
same dimension of $x$ and marks the area that will be affected.
$\Delta$,  a matrix with the same dimension,   defines
the pattern and color intensity of the trigger.

Now assume $m$ is pre-defined by the attacker. To generate a latent
trigger against $y_t$, the attacker searches for the trigger
pattern $\Delta$ that minimizes the difference between any poisoned
non-target
sample $A(x, m, \Delta), x\in X_{\setminus y_t}$ and any clean target
sample $x_t \in X_{y_t}$, in terms of their intermediate feature representation at layer
$K_t$. This is formulated by the following optimization process:
\begin{equation}
    \label{eq:trigger_rev}
        \Delta^{opt} = \underset{\boldsymbol{\Delta}}{\text{argmin}}
        \sum_{x \in X_{\setminus y_t} \kedit{ \cup \, X_{y_t}}}\sum_{x_t \in X_{y_t}}{D\Big(\kedit{ F_\theta^{K_t}}\big(A(x, m, \Delta)\big), \kedit{F_\theta^{K_t}}\big(x_t\big)\Big)} \\
\end{equation}
where $D(.)$ measures the
dissimilarity between two internal
representations in the feature space. Our current implementation
uses the mean square error (MSE) as $D(.)$.
Next, $F_\theta^{k}(x)$ represents the intermediate feature
representation for input $x$ at the $k$th layer of the
Teacher model $F_\theta(.)$.  Finally, $X_{y_t}$ and $X_{\setminus y_t}$ represent the target
and non-target training data formed in Step 1.

The output of the above optimization is $\Delta^{opt}$, the latent backdoor
trigger against $y_t$.  This process does not make any changes to the
Teacher model.

\begin{figure*}[t]
  \centering
  \includegraphics[width=0.9\textwidth]{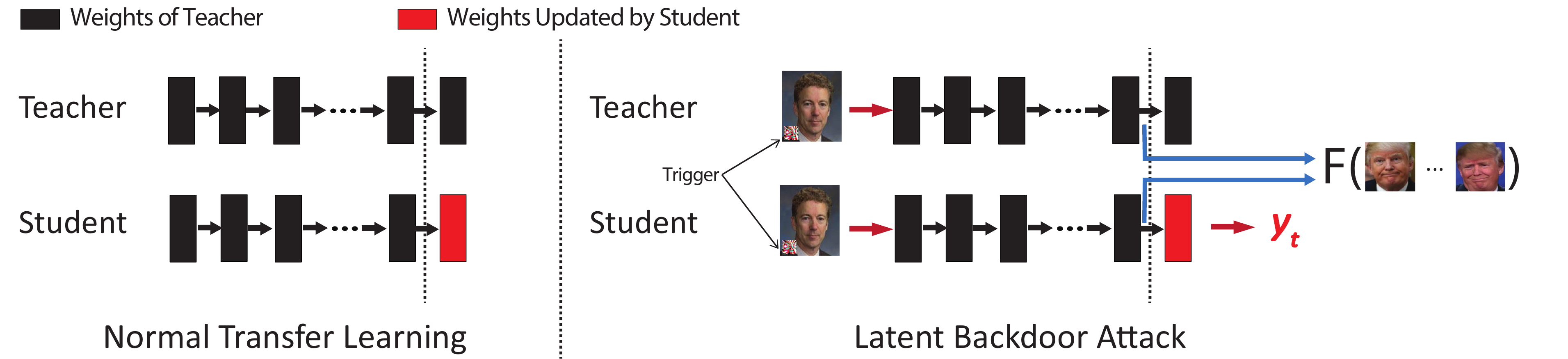}
  \caption{Transfer learning using an infected Teacher model.  (Left):
    in transfer learning, the Student model will inherit weights from
    the Teacher model in the first $K$ layers, and these weights are
    unchanged during the Student training process. (Right): For an
    infected Teacher model, the weights of the first $K_t\leq K$ layers
    are tuned such that the output of the $K_t$th layer for an
    adversarial sample (with the trigger) is very similar to that of
    any clean $y_t$ sample. Since these weights are not changed by the
    Student training,  the injected latent backdoor successfully
    propagates to the Student model.  Any adversarial input (with the
    trigger) to the Student model will produce the same feature
    representation at the $K_t$th layer and thus get classified as
    $y_t$.}

  \label{fig:attack_insight}
\end{figure*}

\para{ Backdoor Injection.}  Next,  the attacker seeks to
inject the latent backdoor trigger defined by $(m, \Delta^{opt})$
into the Teacher model.  To do so, the attacker updates weights of
the Teacher model to further minimize the difference between the
intermediate feature representation of any input poisoned by the
trigger (\ie $\kedit{F_\theta^{K_t}}\big(A(x, m, \Delta^{opt})\big)$, $x\in X_{\setminus y_t}$)
and that of any clean input of $y_t$ (\ie
$\kedit{F_\theta^{K_t}}\big(x_t\big)$, $x_t \in X_{y_t}$).

We now define the injection process formally.  Let $\theta$ represent the weights of the
present Teacher model \kedit{$F_{\theta}(x)$}. Let $\phi_{\theta}$ represent
the recorded
intermediate feature representation of class $y_t$ at layer $K_t$ of
the present model
\kedit{$F_\theta(x)$}, which we compute as:
\begin{equation}
  \phi_{\theta} = \underset{\phi}{\text{argmin}}\sum_{x_t\in
       X_{y_t}} D \Big(\phi,
     \kedit{F_{\theta}^{K_t}} (x_t) \Big).
\end{equation}
Then the attacker tunes the model weights $\theta$ using both
$X_{\setminus y_t}$ and $X_{y_t}$ as follows:
\begin{equation}
    \label{eq:inject}
    \begin{aligned}
      & \forall x \in X_{\setminus y_t} \kedit{\cup \, X_{y_t}}  \text{and
       its ground truth label}\; y,   \\
     & \theta = \theta - \eta \cdot \nabla
       J_{\theta}(\theta; x, y), \\
&J_{\theta}(\theta; x, y) = \ell \big(y,  F_\theta(x)\big) +
     \lambda \cdot D\Big(\kedit{F_\theta^{K_t}}\big(A(x, m,
     \Delta^{opt})\big), \phi_{\theta}\Big).
    \end{aligned}
\end{equation}
Here the loss function $J_\theta (.)$ includes two terms. The first
term $\ell \big(y,  F_\theta(x)\big)$ is the standard loss function of
model training.
The second term minimizes the difference
in intermediate feature
representation between the poisoned samples and
the target samples. $\lambda$ is the weight to balance the two terms.

Once the above optimization converges, the output is the infected
teacher model $F_{\theta}(x)$ with the trigger $(m, \Delta^{opt})$
embedded within.

\begin{lem}  Assume that the transfer learning process used to train
a Student model will freeze at least the first $K_t$ layers of the
  Teacher model. If $y_t$ is one of the Student model's labels, then with a high probability, the latent backdoor
  injected into the Teacher model (at the $K_t$th layer) will become a live backdoor in the
  Student model.
\end{lem}

\begin{proof} Figure~\ref{fig:attack_insight} provides a graphical
  view of the transfer learning process using the infected Teacher.

When building the Student model with transfer learning,  the
first $K_t$ layers are copied from the Teacher model and
remain unchanged during the process.  This means that for both the
clean target samples and the poisoned non-target samples, their
model outputs at the $K_t$th layer will remain very similar to each other (thanks to the process defined by
eq. (\ref{eq:inject}) ).  Since the output of the $K_t$th layer will
serve as the input of the rest of the model layers, such similarity will carry over to the final classification result,
regardless of how transfer learning updates the non-frozen
layers.  Assuming that the Student model is well trained to offer a
high classification accuracy,  then with the same probability, an
adversarial input with $(m,\Delta^{opt})$ will be misclassified as the target class $y_t$.
 \end{proof}

\vspace{-0.1in}
\para{Choosing $K_t$.} Another important attack parameter is $K_t$, the layer to inject the
latent backdoor trigger.  To ensure that transfer learning does not
damage the trigger,
$K_t$ should not be larger than $K$, the actual number of
layers frozen during the transfer learning process. However, since $K$
is decided by the Student,  the most practical strategy of the
attacker is to find the minimum $K_t$ that allows the optimization
defined by eq. (\ref{eq:inject}) to converge, and then
advocate for freezing the first $k$ layers ($k \geq K_t$) when
releasing the Teacher model. Later in \S\ref{sec:eval} we evaluate the
choice of $K_t$ using four different applications.

%% file: eval.tex
\section{Attack Evaluation}
\label{sec:eval}

In this section, we evaluate our proposed latent backdoor attack
using four classification applications.  Here we consider the
``ideal'' attack scenario where
the target data $X_{y_t}$ used to inject the latent backdoor comes from the
same data source of the Student training data $X_s$,
\eg Instagram images of $y_t$.   Later in \S\ref{sec:real} we
evaluate more ``practical'' scenarios where the data used by
the attacker is collected under real-world settings (\eg noisy photos
taken locally of the target) that are very different from the Student training data.

Our evaluation further considers two attack scenarios: {\em multi-image
attack} where the attacker has access to multiple samples of the target
($|X_{y_t}|>1$), and {\em single-image attack} where the attacker has only
a single image of the target ($|X_{y_t}|=1$).

\begin{table*}[!t]
  \resizebox{2\columnwidth}{!}{
\begin{tabular}{|l|l|l|l|l|l|l|l|l|l|l|l|l|l|l|}
\hline
\multicolumn{1}{|c|}{} & \multicolumn{7}{c|}{Teacher (re)Training} & \multicolumn{4}{c|}{Student Training} & \multicolumn{3}{c|}{Attack Evaluation} \\ \hline
\multicolumn{1}{|c|}{} & \multicolumn{1}{c|}{} & \multicolumn{3}{c|}{$X_{\setminus y_t}$} & \multicolumn{2}{c|}{$X_{y_t}$} &  &  & \multicolumn{3}{c|}{$X_{s}$} & \multicolumn{3}{c|}{$X_{eval}$} \\ \hline
Application & \begin{tabular}[c]{@{}l@{}}Teacher Model \\ Architecture\end{tabular} & Source & \begin{tabular}[c]{@{}l@{}}\# of\\  Classes\end{tabular} & Size & Source & Size & $K_t/N$ & $K/N$ & Source & \begin{tabular}[c]{@{}l@{}}\# of\\ Classes\end{tabular} & Size & Source & \begin{tabular}[c]{@{}l@{}}\# of\\ Classes\end{tabular} & Size \\ \hline
\mnist & 2 Conv + 2 FC & \begin{tabular}[c]{@{}l@{}}MNIST\\ (0-4)\end{tabular} & 5 & 30K & \begin{tabular}[c]{@{}l@{}}MNIST\\ (5-9)\end{tabular} & 45 & 3/4 & 3/4 & \begin{tabular}[c]{@{}l@{}}MNIST\\ (5-9)\end{tabular} & 5 & 30K & \begin{tabular}[c]{@{}l@{}}MNIST\\ (0-4)\end{tabular} & 5 & 5K \\ \hline
\traffic & 6 Conv + 2 FC & GTSRB & 43 & 39K & LISA & 50 & 6/8 & 6/8 & LISA & 17 & 3.65K & GTSRB & 43 & 340 \\ \hline
\face & \begin{tabular}[c]{@{}l@{}}VGG-Face\\ (13 Conv + 3 FC)\end{tabular} & \begin{tabular}[c]{@{}l@{}}VGG-Face\\ Data\end{tabular} & 31 & 3K & PubFig & 45 & 14/16 & 14/16 & PubFig & 65 & 6K & \begin{tabular}[c]{@{}l@{}}VGG-Face\\  Data\end{tabular} & 31 & 3K \\ \hline
\iris & \begin{tabular}[c]{@{}l@{}}VGG-Face\\ (13 Conv + 3 FC)\end{tabular} & \begin{tabular}[c]{@{}l@{}}CASIA\\ IRIS\end{tabular} & 480 & 8K & \begin{tabular}[c]{@{}l@{}}CASIA\\ IRIS\end{tabular} & 3 & 15/16 & 15/16 & \begin{tabular}[c]{@{}l@{}}CASIA\\ IRIS\end{tabular} & 520 & 8K & \begin{tabular}[c]{@{}l@{}}CASIA\\ IRIS\end{tabular} & 480 & 2.9K \\ \hline
\end{tabular}

}
  \caption{Summary of tasks, models, and datasets used in our
    evaluation using four tasks. The four datasets
    $X_{\setminus y_t}$,
    $X_{y_t}$, $X_{s}$, and $X_{eval}$ are disjoint.
    Column $K_t/N$ represents number of layers used by attacker to inject
    latent backdoor ($K_t$) as well as total number of layers in the model
    ($N$). Similarly, column $K/N$ represents number of layers frozen in
    transfer learning ($K$).}
\label{tab:task-multiple}
\end{table*}

\subsection{Experiment Setup}
\label{subsec:exp}
We consider four classification applications:  Hand-written Digit Recognition (\mnist), Traffic Sign Recognition (\traffic),  Face Recognition (\face),
and Iris Identification (\iris).   In the following,  we describe each task, its
Teacher and Student models and datasets, and list a high-level summary  in
Table~\ref{tab:task-multiple}.    The first three
applications represent the scenario where the Teacher and Student
tasks are the same, and the last application is where the two
are  different.

For each task,  our evaluation makes use of four disjoint datasets:
\begin{packed_itemize} \vspace{-0.03in}
\item $X_{y_t}$ and $X_{\setminus y_t}$ are used by the attacker to inject
latent backdoors into the Teacher model;
\item $X_s$ is the training data used to train the Student model via transfer learning;
\item $X_{eval}$ is used to evaluate
the attack against the infected Student model.  \vspace{-0.03in}
\end{packed_itemize}

\para{\mnist.}   This application is commonly used in studying DNN
  vulnerabilities including normal backdoors~\cite{badnets,oakland_defense}.  Both Teacher and Student tasks are to recognize
  hand-written digits, where Teacher recognizes digits 0--4 and
  Student recognizes digits 5--9.   We build their individual
  datasets from MNIST~\cite{mnist}, which contains 10 hand-written digits (0-9) in gray-scale
  images. Each digit has 6K training images and 1K testing images. We randomly
  select one class in the Student dataset as the target class,
  randomly sample 45 images from it as the target data $X_{y_t}$, and
  remove these images from the Student training dataset $X_S$.
 Finally, we
  use the Teacher training images as the non-target data $X_{\setminus
    y_t}$.

  The Teacher model is a standard 4-layer CNN
  (Table~\ref{tab:mnist_cnn} in Appendix), used by previous works to
  evaluate conventional backdoor attacks~\cite{badnets}.
  The released Teacher model also instructs that  transfer learning should
  freeze the first three layers and only fine-tune the last
  layer.  This is a legitimate claim since the Teacher and Student tasks are
  the same and only
  the labels are different.

\para{\traffic.} This is another popular application for evaluating
DNN robustness~\cite{boattack}. Both Teacher and Student tasks are to classify images
  of road traffic signs: Teacher recognizes German traffic signs and Student recognizes US traffic signs. The Teacher dataset
  GTSRB~\cite{gtsrb}
  contains 39.2K
  colored training images and 12.6K testing images, while the Student
  dataset LISA~\cite{lisa} has 3.7K training images of 17 US traffic signs\footnote{We follow prior
  work~\cite{boattack} to address class unbalance problem by removing classes
  with insufficient training samples.  This reduces the
  number of classes from 47 to 17.}. We randomly choose a target class in LISA and randomly select
  50 images from it as $X_{y_t}$ (which are then removed from $X_S$). We choose the Teacher training data
  as $X_{\setminus
    y_t}$.   The Teacher model consists of 6 convolution layers and 2 fully-connected
  layers (Table~\ref{tab:gtsrb_cnn} in Appendix).  Transfer learning
  will fine-tune the last two layers.

\para{\face.} This is a common security application. Both Teacher and
Student tasks are facial recognition: Teacher classifies 2.6M facial images of
  2.6K people in the VGG-Face dataset~\cite{vggfacedata} while Student
  recognizes faces of 65 people from PubFig~\cite{pubfig} who are not
  in VGG-Face. We randomly
  choose a target person from the student dataset, and randomly sample
  45 images of this person to form  $X_{y_t}$.  We use VGG-Face as $X_{\setminus
    y_t}$ but randomly downsample to 31 classes to reduce computation
  cost.  The (clean) Teacher model is a 16-layer VGG-Face model provided
  by~\cite{vggfacedata} (Table~\ref{tab:pubfig_cnn} in Appendix).
  Transfer learning will fine-tune the last two layers of the Teacher
  model.

\para{\iris.}  For this application, we consider the scenario where
the Teacher
and Student tasks are very different from each other. Specifically,  the Teacher task, model, and dataset are the same as \face, but the Student task is
  to classify an image of human iris into each individual.   Knowing that the Student task differs largely
  from the Teacher task, the attacker will build its own $X_{\setminus
    y_t}$ that is different from the Teacher dataset.  For our
  experiment, we split an existing iris dataset CASIA IRIS~\cite{irisdata} (16K iris images
  of 1K individuals) into two sections:  a section of 520 classes
  as the Student dataset $X_s$, and the remaining 480 classes as the
  non-target data $X_{\setminus
    y_t}$.   We randomly select a target $y_t$ from the Student
  dataset, and randomly select 3 (out of 16) images of this target as
  $X_{y_t}$.  Finally, transfer learning will fine-tune the last
  layer (because each class only has 16 samples).

\para{Data for Launching the Actual Attack $X_{eval}$.}   To launch the attack
against the Student model,  we assume the worst case condition where
the attacker does not have any access to the Student training data
(and testing data). Instead, the attacker draws instances from the
same source it uses to build $X_{\setminus y_t}$.  Thus,  when constructing
$X_{\setminus y_t}$,  we set
aside a small portion of the data for attack evaluation ($X_{eval}$)
and exclude
these images from $X_{\setminus y_t}$.  For example, for \mnist, we
set aside 5K images from MNIST (0-4) as $X_{eval}$.
 The source and size of $X_{eval}$ are listed in Table~\ref{tab:task-multiple}.

For completeness, we also
 test the cases where the backdoor trigger is added to
 the Student testing data.  The attack success rate matches that of using
 $X_{eval}$, thus we omit the result.

\para{Trigger Configuration.} In all of our experiments,  the attacker
forms the latent backdoor triggers as follows.
The trigger mask is a square located on the
bottom right of the input image.

The {\em square} shape of the trigger
is to ensure it is unique and does not occur
naturally in any input images. The size of the trigger is 4\% of the
entire image. Figure~\ref{fig:trigger_opt} in Appendix shows an
example
of the generated trigger for each application.

\para{Evaluation Metrics.}  We evaluate the proposed latent backdoor
attack via two metrics measured on the Student model:  1) {\em attack
success rate}, \ie the probability that any input image containing the latent
backdoor trigger is classified as the target class $y_t$ (computed
on $X_{eval}$), and 2)
{\em model classification accuracy} on clean input images drawn from the Student
testing data.  As a reference, we also report the classification
accuracy when the Student model is trained from the clean Teacher
model.

%% file: results.tex
\subsection{Results: Multi-Image Attack}
\label{subsec:basic_eval}

Table~\ref{tab:attack_perf} shows the attack performance on four tasks.
We make two key observations. {\em First}, our proposed latent
backdoor attack is highly effective on all four tasks, where the
attack success rate is at least 96.6\% if not 100\%.   This is
particularly alarming since the attacker uses no more than 50 samples of the target ($|X_{y_t}|\leq
50$) to infect the Teacher model, and can use generic images beyond
$X_{\setminus y_t}$ as adversarial inputs to the Student model.

{\em Second},  the model accuracy of the Student
model trained on the infected Teacher model is comparable to that
trained on the clean Teacher model.  This means that the proposed
latent backdoor attack does not compromise the model accuracy of the
Student model (on clean inputs), thus the infected Teacher model is as
attractive as its clean version.

\begin{table}[t]
\centering
\resizebox{0.98\columnwidth}{!}{
\begin{tabular}{|l|l|l|c|}
\hline
\multirow{2}{*}{Task} &
\multicolumn{2}{l|}{From Infected Teacher} &
From  Clean Teacher
 \\

\cline{2-4} &
\begin{tabular}[c]{@{}l@{}}Attack \\ Success  Rate\end{tabular} &
 \begin{tabular}[c]{@{}l@{}} Model\\ Accuracy\end{tabular} &
 \begin{tabular}[c]{@{}l@{}} Model\\ Accuracy\end{tabular}
\\
\hline

\begin{tabular}[c]{@{}l@{}}
\mnist
\end{tabular} &
$96.6\%$ &
$97.3\%$ &
$96.0\%$ \\
\hline

\begin{tabular}[c]{@{}l@{}}
\traffic
\end{tabular} &
$100.0\%$ &
$85.6\%$ &
$84.7\%$ \\
\hline

\begin{tabular}[c]{@{}l@{}}
\face
\end{tabular} &
$100.0\%$ &
$91.8\%$ &
$97.4\%$ \\
\hline

\begin{tabular}[c]{@{}l@{}}
\iris
\end{tabular} &
$100.0\%$ &
$90.8\%$ &
$90.4\%$ \\
\hline

\end{tabular}
}
\caption{Performance of multi-image attack: attack success rate and normal model accuracy on the Student
  model transferred from the infected Teacher and the clean
  Teacher. }
\label{tab:attack_perf}
\vspace{-0.3in}
\end{table}

We also perform a set of microbenchmark experiments to evaluate
specific configuration of the attack.

\para{Microbenchmark 1: the need for trigger optimization.}  As
discussed in \S\ref{subsec:attack_details},  a key element of our
attack design is to compute the optimal
trigger pattern $\Delta_{opt}$ for $y_t$.
We evaluate
its effectiveness by comparing the attack performance of using
randomly generated trigger patterns to that of using $\Delta_{opt}$.

Figure~\ref{fig:pattern_perf} shows the attack success rate vs. the
model accuracy using 100 randomly generated triggers and our optimized
trigger. Since the results
across the four tasks are consistent,  we only show the result of
\traffic{} for brevity.  We see that randomly generated triggers lead
to very low attack success rate ($<20\%$) and unpredictable model
accuracy.  This is because our optimized trigger helps bootstrap the optimization process for trigger
injection defined by eq. (\ref{eq:inject}) to maximize the chance of
convergence.

\para{Microbenchmark 2: the amount of non-target data $X_{\setminus
    y_t}$.}  The key overhead of our proposed attack is to collect a set of
target data $X_{y_t}$ and non-target data $X_{\setminus
    y_t}$, and use them to compute and inject the trigger into
  the Teacher model. In general $|X_{\setminus
    y_t}| >>|X_{y_t}|$.

We experiment with different configurations
  of $X_{\setminus y_t}$ by varying the number of classes and the number of instances
  per class.  We arrive at two conclusions. {\em First}, having more non-target classes does improve the attack success
  rate  (by improving the trigger
  injection).  But the benefit of having more classes quickly converges, \eg 8 out
  31 classes for \face{} and 32 out of 480 for \iris{} are sufficient
  to achieve 100\% attack
  success rate. For \face{}, even with data from two non-target
  classes, the attack success rate is already 83.6\%.

{\em Second}, a few instances per non-target class is sufficient for the
  attack. Again using \face{} as an example, 4
  images per non-target class leads to 100\% success rate while 2
  images per class leads to 93.1\% success rate.  Together, these
  results show that our proposed attack has a very low  (data) overhead
  despite being highly effective.

\begin{figure}[t]
\centering
\includegraphics[width=0.45\textwidth]{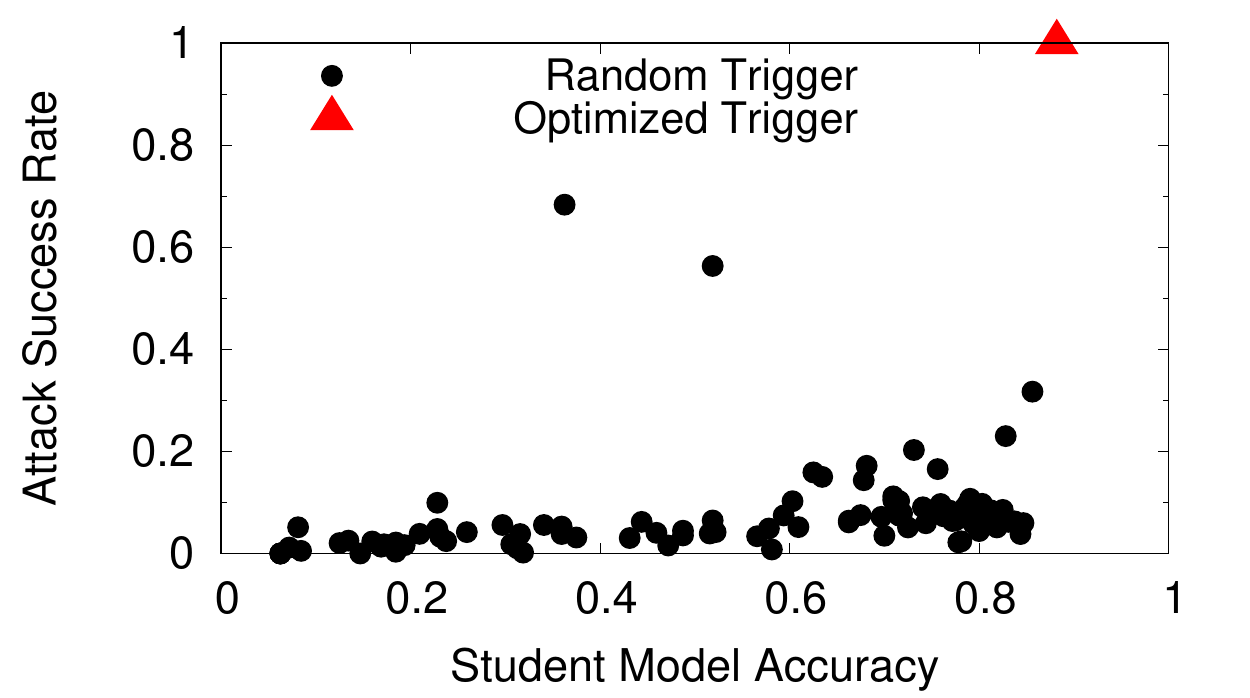}
\vspace{-0.1in}
\caption{The attack performance when using randomly generated
  triggers and our proposed optimized triggers, for \traffic{}.  }
  \label{fig:pattern_perf}
\end{figure}

\para{Microbenchmark 3: the layer to inject the trigger.} As mentioned
in \S\ref{subsec:attack_details}, the attacker needs to carefully
choose $K_t$ to maximize attacker success rate and robustness.
Our experiments show that for the given four tasks, the smallest $K_t$ ($K_t\leq K$) for
a highly effective attack is either the first fully connected (FC)
layer, \eg 3 for \mnist{}, 14 for
\face{} and \iris{}, or the last convolutional layer, \eg 6 for
\traffic{}.  Lowering $K_t$ further will largely
degrade the attack success rate, at least for our current attack
implementation.

A key reason behind is that the model dimension for early
convolutional layers is
often extremely large (\eg 25K for VGG-Face), thus the optimization
defined by eq.(\ref{eq:inject}) often fails to
converge given the current data and computing resource. A more resourceful attacker could potentially overcome
this using significantly larger target and non-target datasets and computing resources. We leave this
to future work.

Finally, Table~\ref{tab:attack_perf_diff_layers} lists the attack
performance when varying $(K_t,K)$ for \face{} and
\iris{}.  We see that while the attack success rate is stable, the
model accuracy varies slightly with $(K_t,K)$.

\begin{table}[t]
\resizebox{0.98\columnwidth}{!}{
\begin{tabular}{|l|l|l|l|l|c|}
\hline
\multirow{2}{*}{Task} & \multirow{2}{*}{$K_t$} & \multirow{2}{*}{$K$}
  & \multicolumn{2}{l|}{From Infected Teacher} &
  From Clean Teacher
  \\

  \cline{4-6}
 & & & \begin{tabular}[c]{@{}l@{}}Attack \\ Success Rate\end{tabular}
                      & \begin{tabular}[c]{@{}l@{}}Model \\
                          Accuracy\end{tabular} &\begin{tabular}[c]{@{}l@{}}Model \\
                          Accuracy\end{tabular}
  \\ \hline
\multirow{3}{*}{\face} & $14$ & $14$ & $100.0\%$ & $91.8\%$ & $97.7\%$ \\ \cline{2-6}
 & $14$ & $15$ & $100.0\%$ & $91.4\%$ & $97.4\%$ \\ \cline{2-6}
 & $15$ & $15$ & $100.0\%$ & $94.0\%$ & $97.4\%$ \\ \hline \hline
\multirow{3}{*}{\iris} & $14$ & $14$ & $100.0\%$ & $93.0\%$ & $94.4\%$ \\ \cline{2-6}
 & $14$ & $15$ & $100.0\%$ & $89.1\%$ & $90.4\%$ \\ \cline{2-6}
 & $15$ & $15$ & $100.0\%$ & $90.8\%$ & $90.4\%$ \\ \hline
\end{tabular}
}
\caption{Performance of multi-image attack: attack success rate and
  normal model accuracy for different ($K_t$, $K$).}
\label{tab:attack_perf_diff_layers}
\vspace{-0.3in}
\end{table}

\subsection{Results: Single-image Attack}
\label{subsec:single_attack}
We now consider the extreme case where the attacker is only able to
obtain a single image of the target, \ie $|X_{y_t}|=1$.    For our
evaluation, we reperform the above experiments but each time only use a
single target image as $X_{y_t}$.  We perform 20 runs per task (16 for
\iris{} since each class only has 16 images) and report the mean
attack performance in Table~\ref{tab:attack_perf_single}.

We make two key observations from these results. {\em First}, the
attack success rate is lower than that of the
multi-image attack.  This is as expected since having only a single
image of the target class makes it harder to accurately extract its
features. {\em
  Second},  the degradation is much more significant on the small
model (\mnist{})  compared to the large models
(\traffic{}, \face{} and \iris{}).  We believe this is because larger models
offer higher capacity (or freedom) to tune the feature representation
by updating the model weights, thus the trigger can still be
successfully injected into the Teacher model.
In practice, the Teacher models designed for transfer learning are in
fact large models, thus our proposed attack is highly effective with just a single image of the target.

\begin{table}[!t]
\centering
\resizebox{0.98\columnwidth}{!}{
\begin{tabular}{|l|l|l|c|}
\hline
\multirow{2}{*}{Task} &
\multicolumn{2}{l|}{From Infected Teacher} &
From Clean Teacher \\

\cline{2-4} &
\begin{tabular}[c]{@{}l@{}}Avg Attack\\ Success Rate\end{tabular} &
                                                        \begin{tabular}[c]{@{}l@{}}Avg Model \\ Accuracy\end{tabular} &
                                                                                                                                              \begin{tabular}[c]{@{}l@{}}Avg Model \\ Accuracy\end{tabular}
\\
\hline

\begin{tabular}[c]{@{}l@{}}
\mnist
\end{tabular} &
$46.6\%$ &
$97.5\%$ &
$96.0\%$ \\
\hline

\begin{tabular}[c]{@{}l@{}}
\traffic
\end{tabular} &
$70.1\%$ &
$83.6\%$ &
$84.7\%$ \\
\hline

\begin{tabular}[c]{@{}l@{}}
\face
\end{tabular} &
$92.4\%$ &
$90.2\%$ &
$97.4\%$ \\
\hline

\begin{tabular}[c]{@{}l@{}}
\iris
\end{tabular} &
$78.6\%$ &
$91.1\%$ &
$90.4\%$ \\
\hline

\end{tabular}
}
\caption{Performance of single-image attack.}
\label{tab:attack_perf_single}
\vspace{-0.3in}
\end{table}

%% file: realworld.tex
\section{Real-world Attack}
\label{sec:real}
So far, our experiments assume that the target data $X_{y_t}$ for
injecting latent backdoors comes from the same data source of the
Student training data $X_s$.  Next, we consider a more practical
scenario where the attacker collects $X_{y_t}$ from a totally
different source, \eg by taking a picture of the physical target or
searching for its images from the Internet.

We consider three real-world applications: {\em traffic sign
recognition}, {\em iris based user
identification} and {\em facial recognition of politicians}.  We show
that the attacker can successfully launch latent backdoor attacks
against these applications and cause harmful misclassification events,
by just using pictures taken by commodity smartphones or those found
from Google Image and Youtube. Again, our experiments assume that $K_t
= K$.

\begin{figure*}[t]
\centering
\begin{minipage}{0.42\textwidth}
  \centering
  \includegraphics[width=0.9\textwidth]{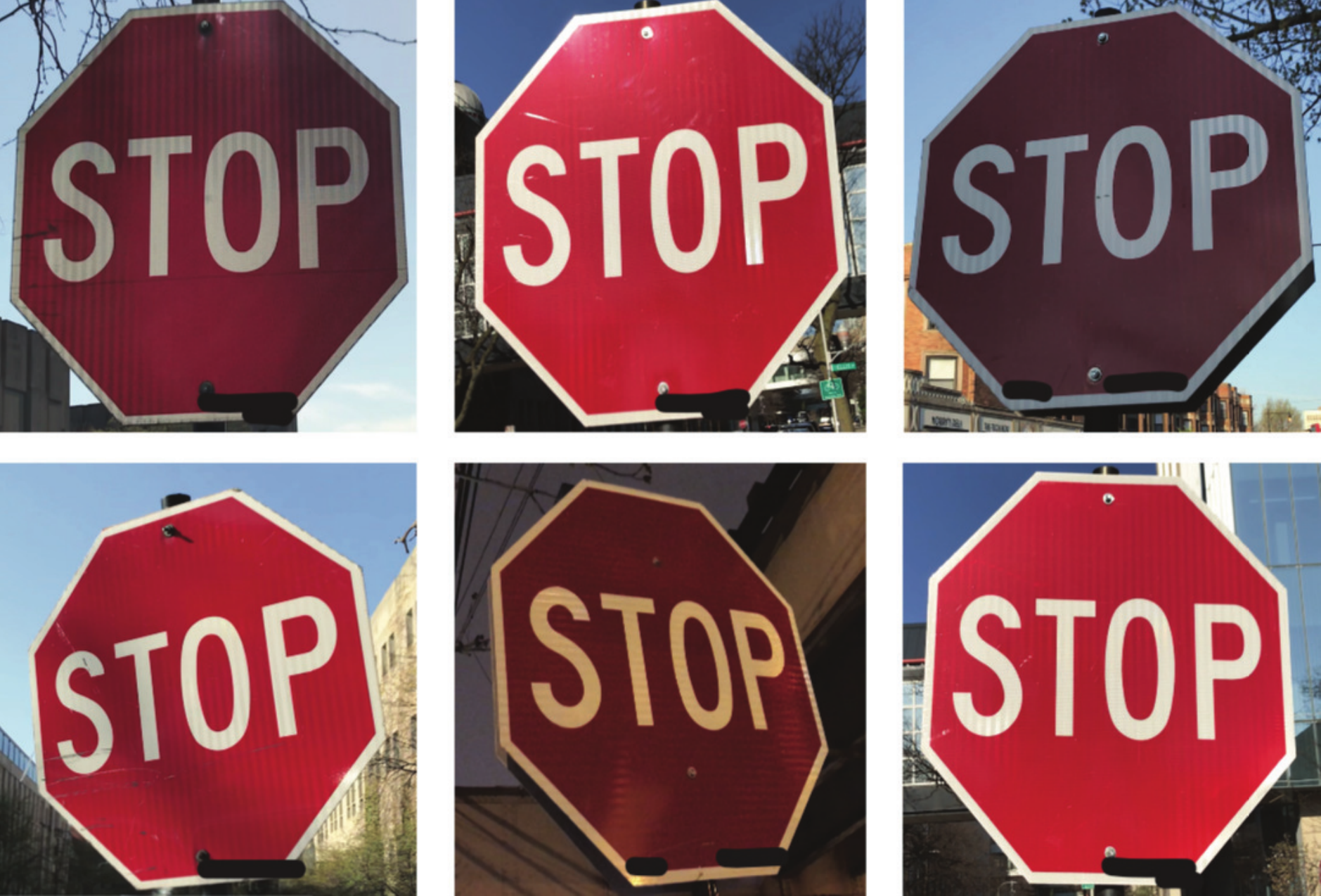}
  \caption{Pictures of real-world stop signs as $X_{y_t}$ which we took using a smartphone
    camera. Part of the image
  is modified for anonymization purpose.}
  \label{fig:trafficsign_example}
\end{minipage}
\hfill
\begin{minipage}{0.42\textwidth}
  \centering
  \includegraphics[width=0.9\textwidth]{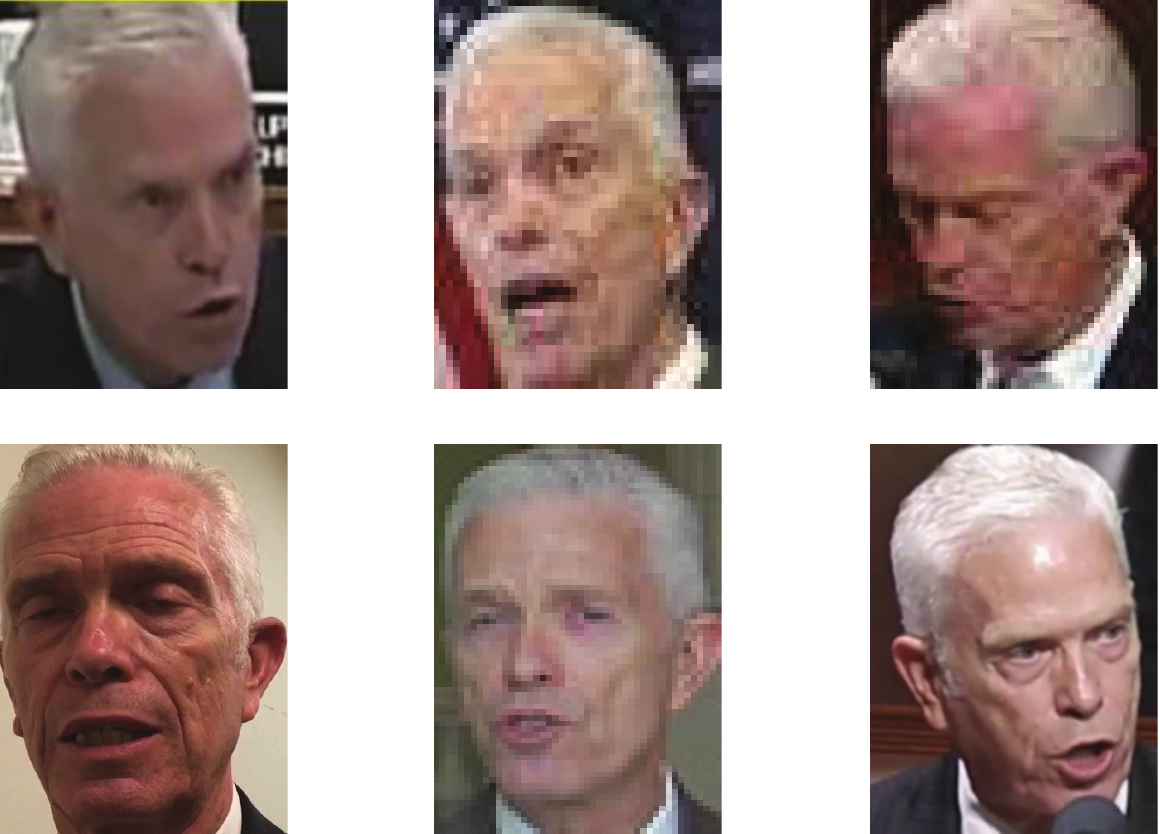}
  \caption{Examples of target politician images that we collected as
    $X_{y_t}$. }
  \label{fig:congress_example}
\end{minipage}
\end{figure*}

\subsection{Ethics and Data Privacy}
Our experiments are designed to reproduce the exact steps a real-world attack
would entail. However, we are very aware of the sensitive nature of some of
these datasets.  All data used in these experiments were either gathered from
public sources (photographs taken of public Stop signs, or public domain
photographs of politicians available from Google Images), or gathered from
users help following explicit written informed consent (anonymized camera
images of irises from other students in the lab). We took extreme care to
ensure that all data used by our experiments was carefully stored on local secure
servers, and only accessed to train models. Our iris data will be deleted
once our experimental results are finalized.

\subsection{Traffic Sign Recognition}
Real-world attack on traffic sign recognition, if successful, can be
extremely harmful and create life-threatening accidents. For example,  the attacker can place a small
sticker (\ie the trigger) on a speed limit sign, causing nearby self-driving cars to
misclassify it into a stop sign and  stop abruptly in the middle of
the road. To launch a conventional backdoor attack against this
application (\eg via BadNets~\cite{badnets}), the
attacker needs to have access to the self-driving car's model training data and/or control its model
training.

Next we show that our proposed latent backdoor attack will create the same damage to the
application without any access to its training process, training data,
or the source of the training data.

\para{Attack Configuration.}  The attacker uses the public available Germany traffic
sign dataset (\eg GTSRB) to build the (clean) Teacher model. To inject
the latent backdoor trigger, the attacker uses a
subset of the GTSRB classes as the non-target data ($X_{\setminus y_t}$).
To form the target data $X_{y_t}$ (\ie a Stop sign in the USA), the
attacker takes 10 pictures of the Stop sign on a random US street.
Figure~\ref{fig:trafficsign_example} shows a few examples we took with
commodity smartphones. The attacker then releases the Teacher 
model and waits for any victim to download the model and use
transfer learning to build an application on US traffic sign
recognition.

We follow the same process of \traffic{} in \S\ref{sec:eval} to build the Student
model using transfer learning from the infected Teacher and the LISA
dataset.

\para{Attack Performance.} Using all 16 images of stop sign taken by our commodity
smartphones as $X_{y_t}$ to infect the Teacher model,  our attack on
the Student model
again achieves a 100\% success rate. Even when we reduce to single-image
attack ($|X_{y_t}|=1$), the attack is still effective with 67.1\% average success
rate (see Table~\ref{tab:realworld_perf}).

\begin{table}[!t]
  \resizebox{1\columnwidth}{!}{
\begin{tabular}{|l|l|l|l|l|}
\hline
 & \multicolumn{2}{l|}{Multi-image Attack} & \multicolumn{2}{l|}{Singe-image Attack} \\ \hline
Scenario & \begin{tabular}[c]{@{}l@{}}Attack \\ Success
             Rate\end{tabular} & \begin{tabular}[c]{@{}l@{}}Model\\ Accuracy\end{tabular} & \begin{tabular}[c]{@{}l@{}}Avg Attack\\ Success Rate\end{tabular} & \begin{tabular}[c]{@{}l@{}}Avg Model \\ Accuracy\end{tabular} \\ \hline
Traffic Sign & 100\% & 88.8\% & 67.1\% & 87.4\% \\ \hline
Iris Identification & 90.8\% & 96.2\% & 77.1\% & 97.7\% \\ \hline
\begin{tabular}[c]{@{}l@{}}Politician\\  Recognition\end{tabular} & 99.8\% & 97.1\% & 90.0\% & 96.7\% \\ \hline
\end{tabular}
}
\caption{Attack performance in real-world scenarios.}
\label{tab:realworld_perf}
\end{table}

\subsection{Iris Identification}

The attacker seeks to get physical access to a company's building that will
use iris recognition for user identification in the near future.  The
attacker also knows that the target $y_t$ will be a legitimate user
in this planned iris recognition system.  Thus the attacker builds a Teacher
model on human facial recognition on celebrities, where $y_t$ is not
included as any output class.  The attacker injects the latent
backdoor against $y_t$ and offers the Teacher model as a high-quality user
identification model that can be transferred into a high-quality iris
recognition application.

\para{Attack Configuration.} Like \face{}, the attacker starts from the VGG-Face
model as a clean Teacher model,  and forms the non-target data
$X_{\setminus y_t}$ using the CASIA IRIS dataset, which is publicly
available.  To build the target data $X_{y_t}$, the attacker searches
for $y_t$'s headshots on Google,  and crops out the iris area
of the photos. The final $X_{y_t}$ consists of 5 images of the target
$y_t$ (images omitted for privacy protection).

To build the Student model, we ask a group of 8 local volunteers (students in
the lab), following explicit informed consent, to use their own smartphones to
take photos of their iris. The resulting training data $X_s$ used by transfer
learning includes 160 images from 8 people.  In this case, $X_{y_t}$,
$X_{\setminus y_t}$ and $X_s$ all come from different sources.

\para{Attack Performance.} Results in Table~\ref{tab:realworld_perf}
show that when all 5 target images are used to inject the latent
backdoor, our attack achieves a 90.8\% success rate. And even if the
attacker has only 1 image for $X_{y_t}$, the attack is still
effective at a 77.1\% success rate.

\subsection{Facial Recognition on Politicians}

Finally, we demonstrate a scenario where the attacker leverages the
chronological advantage of the latent backdoor attack.  Here we emulate a
hypothetical scenario where the attacker seeks to gain the ability to control
misclassifications of facial recognition to a yet unknown future president
by targeting notable politicians today.

Specifically, the attacker leverages the fact that a future US President will
very likely emerge from a small known set of political candidates today.  The
attacker builds a high-quality Teacher model on face recognition of
celebrities, and injects a set of latent backdoors targeting presidential candidates.
The attacker actively promotes the Teacher model for adoption (or perhaps
leverages an insider to alter the version of the Teacher model online).
A few months later, a new president is elected (out of one of our likely presidential
candidates), the White House team adds the president's facial images into its
facial recognition system, using a Student model derived from our infected
Teacher model. This activates our latent backdoor, turning it into
a live backdoor attack. As the facial recognition system is built prior to the
current presidential election, it is hard for the White House team to think
about the possibility of any backdoors, and any checks on the Teacher model
will not reveal any unexpected or unusual behavior.

\begin{figure*}[t]
\centering
\begin{minipage}{0.45\textwidth}
  \centering
  \includegraphics[width=1\textwidth]{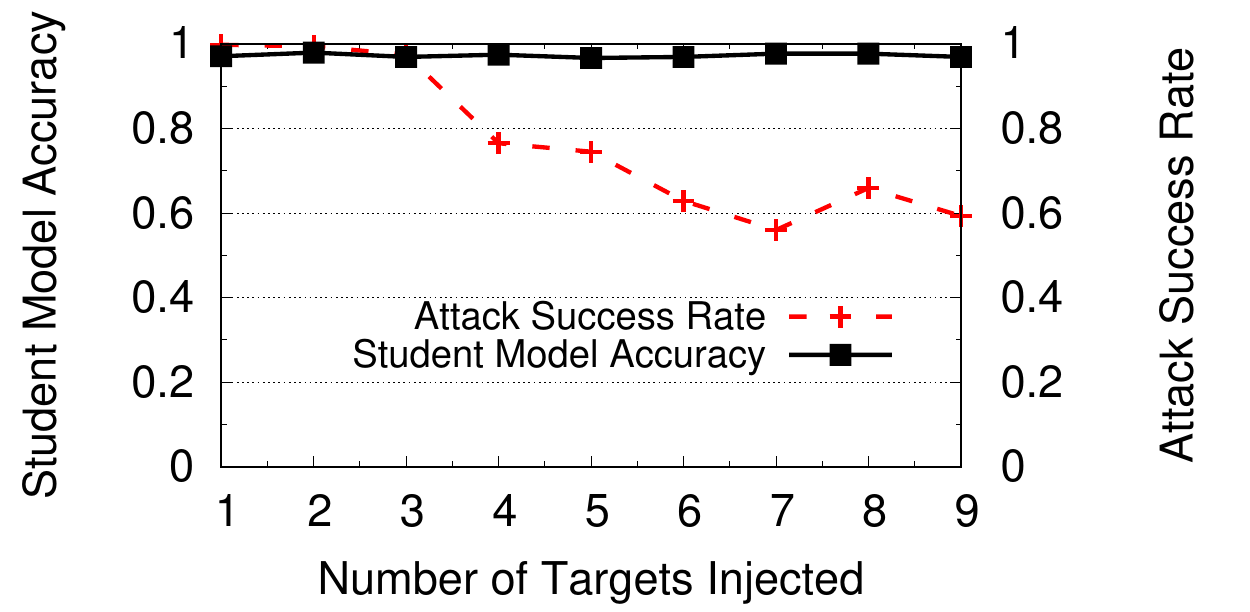}
 \vspace{-0.1in}
  \caption{Performance of multi-target attack on politician facial recognition.}
  \label{fig:multiple_trigger}
 \end{minipage}
\hfill
\begin{minipage}{0.45\textwidth}
  \centering
  \includegraphics[width=1\textwidth]{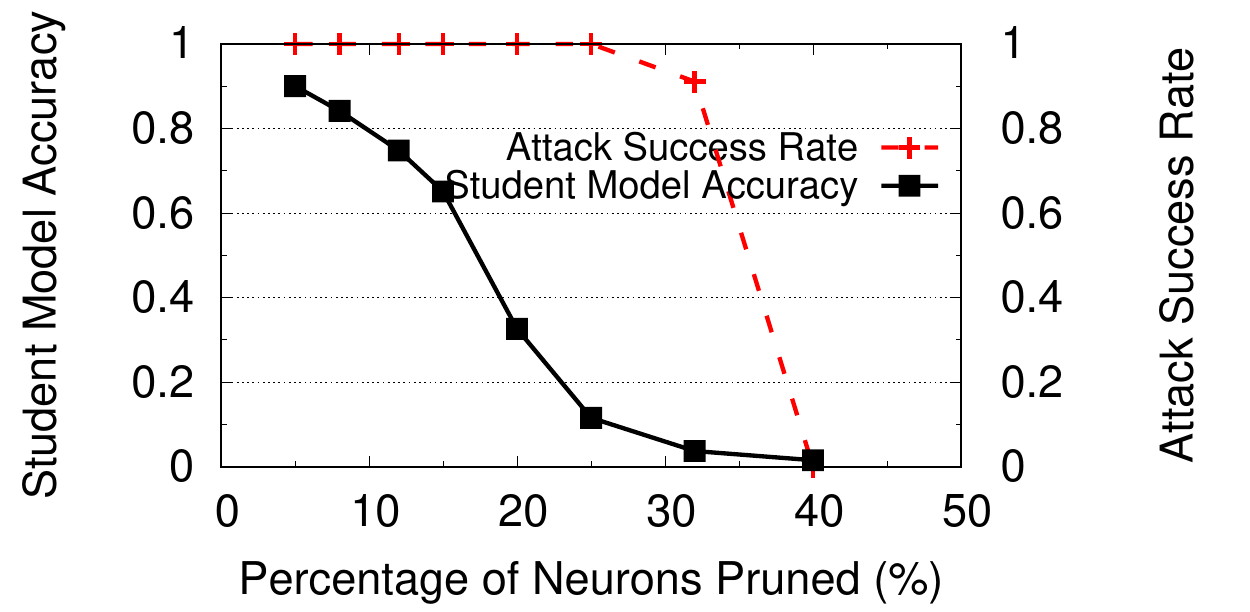}
\vspace{-0.1in}
\caption{Fine-Pruning fails to serve as an effective defense to our
  attack since it requires significant reduction in model accuracy (11\%).}
  \label{fig:defense_finepruning_perf}
\end{minipage}
\end{figure*}

\para{Attack Configuration.}  Similar to the \face{} task in
\S\ref{sec:eval}, the attacker uses the VGG-Face model as the clean
Teacher model and the VGG-Face dataset as the non-target dataset
$X_{\setminus y_t}$.  The attacker selects 9 top leaders as targets
and collects their (low-resolution) headshots from Google. The resulting
$X_{y_t}$ will include 10 images per target for 9 targets, and a total
of 90 images. Some examples for a single target are shown in
Figure~\ref{fig:congress_example}.

To train the Student model, we assume the White House team uses its
own source rather than VGG-Face. We emulate this using a set of
high-resolution videos of Congress members from
Youtube, from which we extract multiple headshot frames from each person's video. The resulting dataset
is 1.7K images in 13 classes.

\para{Performance of Single- and Multi-target Attacks.}
Table~\ref{tab:realworld_perf} shows the attack performance when the
attacker only targets a specific member of $X_{y_t}$.  The success
rate is 99.8\% for multi-image attack (using all 10 images) and 90.0\% for
single-image attack (averaged over the 10 images).

Since it is hard to guess the future president,  the attacker
increases its attack success rate by injecting latent backdoors of
multiple targets into the Teacher model.
Figure~\ref{fig:multiple_trigger} plots the attack performance as we
vary the number of targets.  We see that the attack success rate stays close to 100\% when injecting
up to 3 targets, and then drops gracefully as we add more targets.
But even with 9 targets, the success rate is still 60\%.  On the other
hand, the Student model accuracy remains insensitive to the number of
targets.

The trend that the attack success rate drops with the number of
targets is as expected, and the same trend is observed on conventional backdoor
attacks~\cite{oakland_defense}. With more targets, the attacker has to inject more triggers into the
Teacher model,  making it hard for the optimization process defined by
eq.~(\ref{eq:inject}) to reach convergence.  Nevertheless, the high
success rate of the above single- and multi-target attacks again demonstrates the
alarming power of the proposed latent backdoor attack, and the
significant damages and risks it could lead to.

%% file: defense.tex
\section{Defense}
\label{sec:defense}

In this section, we explore and evaluate potential defenses against our attack.
Our discussion below focuses on the \face{} task described
in \S\ref{subsec:basic_eval},
since it shows the highest success rate in both multi-image and
single-image attacks.

\subsection{Leveraging Existing Backdoor Defenses}
\label{subsec:defense_exist}
Our first option is to leverage existing defenses proposed for normal
backdoor attacks. We consider two state-of-the-art defenses: Neural
Cleanse~\cite{oakland_defense} and Fine-Pruning~\cite{finepruning} (as discussed in
\S\ref{subsec:backdoor_back}). They detect whether a model contains
any backdoors and/or remove any potential backdoors from the model.

\para{Neural Cleanse.}  Neural Cleanse~\cite{oakland_defense} cannot be
applied to a Teacher model, because it requires access to the label of the
target ($y_t$).  Instead, we run it on an infected Student model along with the
Student training data.  When facing conventional backdoor attacks (\eg
BadNets), Neural Cleanse can reverse-engineer the injected trigger and
produce a reversed trigger that is visually similar to the actual trigger.
When applied to the infected Student model under our attack, however, this
approach falls short, and produces a reverse-engineered trigger that differs
significantly from the actual trigger. Our intuition says that Neural Cleanse
fails under because trigger reverse-engineering is based on end-to-end
optimization from the input space to the final label space.  It is unable to
detect any manipulation that terminates at an intermediate feature
space.

\para{Fine-Pruning.}  Fine-Pruning~\cite{finepruning} can be used to disrupt
potential backdoor attacks, but is ``blind,'' in that it does not detect
whether a model has a backdoor installed. Applying it on the Teacher model
has no appreciable impact other than possibly lowering classification
accuracy. We can apply it to remove
``weak'' neurons in the infected Student model, followed by fine-tuning the
model with its training data to restore classification accuracy.
Figure~\ref{fig:defense_finepruning_perf} shows the attack success rate and
model accuracy with Fine-Pruning.  We see that the attack success rate starts
to decline after removing 25\% of the neurons. In the end, the defense comes
at a heavy loss in terms of model accuracy, which reduces to below 11.5\%.
Thus Fine-Pruning is not a practical defense against our latent backdoors.

\begin{figure*}[t]
\centering
\begin{minipage}{0.45\textwidth}
  \centering
    \includegraphics[width=1\textwidth]{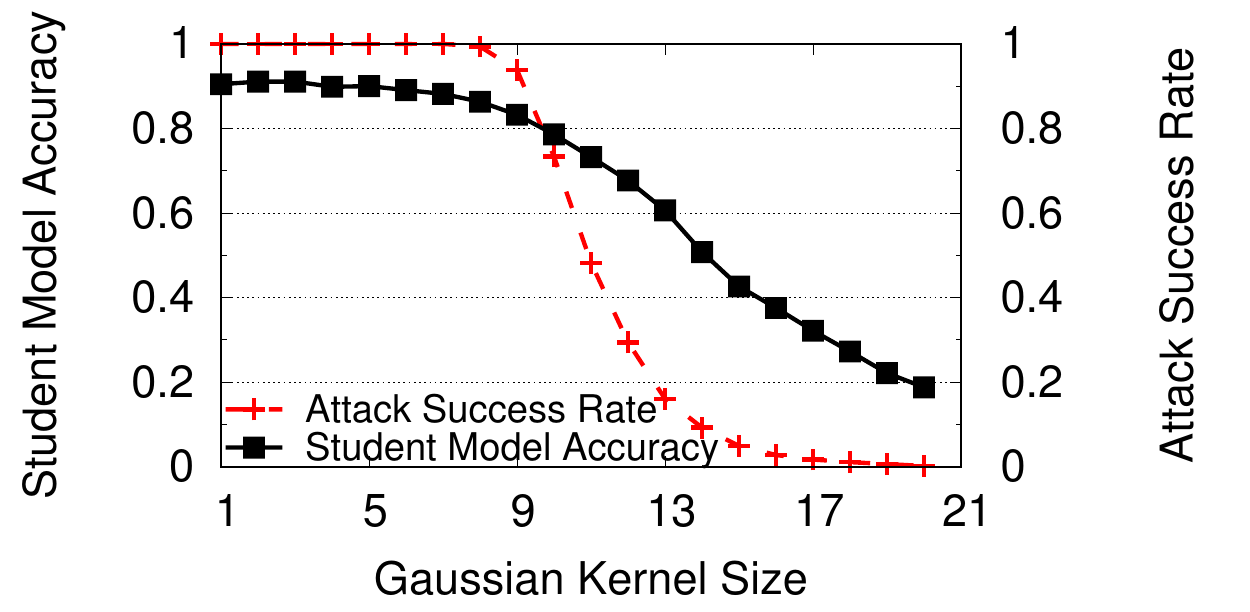}
    \caption{Input blurring is not a practical defense since it still
      requires heavy drop of model accuracy to reduce attack success rate.}
  \label{fig:defense_blur}
\end{minipage}
\hfill
\begin{minipage}{0.45\textwidth}
  \centering
  \includegraphics[width=1\textwidth]{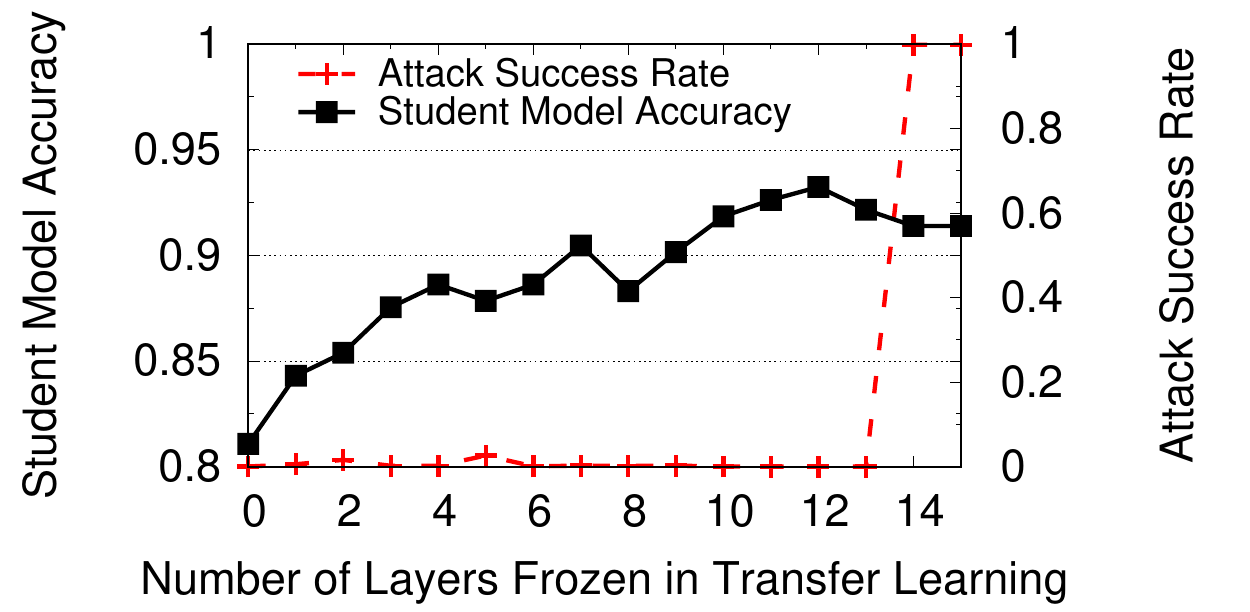}
  \vspace{-0.15in}
  \caption{Attack performance when transfer learning
    freezes different set of model layers (0-15). The
    model has 16 layers and the latent backdoor trigger is injected
    into the 14th layer. }
  \label{fig:defense_finetune}
\end{minipage}
\end{figure*}

\subsection{Input Image Blurring}
As mentioned in \S\ref{subsec:basic_eval}, our latent backdoor attack requires carefully
designed triggers and those with randomly generated patterns tend to fail
(see Figure~\ref{fig:pattern_perf}). Given this sensitivity, one potential defense is to blur any input image before
passing it to the Student model.  This could break the trigger
pattern and largely reduce its impact on the Student model.

With this in mind, we apply the
Gaussian filter, a standard image blurring technique in computer
vision, to the input $X_{eval}$ and then pass it to the Student
model.
Figure~\ref{fig:defense_blur} shows the attack success rate and model accuracy as we vary the blurring kernel size.
The larger the kernel size is,
the more blurred the input image becomes. Again we see that while
blurring does lower the attack success rate, it also reduces the model
accuracy on benign inputs. Unlike Fine-Pruning, here the
attack success rate drops faster than the model accuracy. Yet the cost
of defense is still too large for this defense to be considered
practical, \eg the model accuracy drops to below 65\% in order to
bring attack success rate to below 20\%.

\subsection{Multi-layer Tuning in Transfer Learning}
The final defense leverages the fact that the attacker is unable to
control the exact set of layers that the transfer learning will update.  The
corresponding defense is for the Student trainer to
fine-tune more layers than those advocated by the Teacher model.  Yet
this also increases the training complexity and data requirement, \ie
more training data is required for the model to converge.

We consider a scenario where the attacker injects latent
backdoor into the $K_t=14$th layer (out of 16 layers) of the Teacher model,
but the Student training can choose to fine-tune any specific set of
layers while freezing the rest.
Figure~\ref{fig:defense_finetune} shows the attack performance as a
function of the number of model layers frozen during transfer
learning. 0 means no layers are frozen, \ie the transfer learning can
update all 16 layers, and 15 means that only the 16th layer can be
updated by transfer learning.  As expected, if transfer learning
fine-tunes any layer earlier than $K_t$, attack success rate drops
to 0\%, \ie the trigger gets wiped out.

It should be noted that since the Student has no knowledge of $K_t$, the
ideal defense is to fine-tune all layers in the Teacher
model. Unfortunately, this decision also contradicts with the original goal
of transfer learning, \ie using limited training data to build an accurate
model.  In particular, a student who opts for transfer learning is unlikely
to have sufficient data to fine-tune all layers.  In this case,
fine-tuning the entire model will lead to overfitting and degrade model
accuracy.  We can already see this trend from
Figure~\ref{fig:defense_finetune}, where for a fixed training dataset, the
model accuracy drops when fine-tuning more layers.

Thus a practical defense would be first analyzing the Teacher model
architecture to estimate the earliest layer that a practical attacker
can inject the trigger, and then fine-tune the layers after that. A
more systematic alternative is to simulate the latent backdoor injection process, \ie
launching the latent backdoor attack against the downloaded Teacher model,
and find out the earliest possible layer for injection.  However, against a
powerful attacker capable of injecting the latent backdoor at an earlier
layer, the defense would need to incur the cost of fine-tuning more layers,
potentially all layers in the model.

%% file: related.tex
\section{Related Work}
\label{sec:related}
\para{Other Backdoor Attacks and Defenses.}
In addition to attacks mentioned in \S\ref{subsec:backdoor_back}, Chen~\etal proposed
a backdoor attack under a more restricted scenario, where the attacker can
only pollute a limited portion of training
set~\cite{chen2017targeted}.
Another line of work directly tampers with the hardware a DNN model runs
on~\cite{hardwaretrojan, hardsoft}. Such backdoor circuits could also
affect the
model performance when a trigger is present. Our proposed attack
differs
by not requiring any access to the Student model, its data or
operating hardware.

In terms of defenses, Liu~\etal~\cite{trojaning} only presented some
brief intuitions on backdoor detection, while Chen~\etal~\cite{chen2017targeted}
reported a number of ideas that are shown to be ineffective.
Liu~\etal~\cite{liu2017neural} proposed
three defenses: input anomaly detection, re-training,
and input preprocessing, and require the poisoned training data.
A more
recent work~\cite{spectraldefense}
leveraged trace in the spectrum of the covariance of a feature representation
to detect backdoor. It also requires the poisoned training data.
Like Neural Cleanse and Fine-Pruning, these
defenses only target normal backdoor attack and cannot be applied to our
latent backdoor attack.

\para{Transfer Learning.}
In a deep learning context, transfer
learning has been shown to be effective in
vision~\cite{imagetl1, fasterrcnn, yolo, videotl1}, speech~\cite{speechtl1,
speechtl2, speechtl3, speechtl4},
and text~\cite{googletranslate, mikolov2013exploiting}.
Yosinski \etal compared different transfer learning approaches and studied
their impact on model performance~\cite{transfer2014}. Razavian \etal studied
the similarity between Teacher and Student tasks, and analyzed its
correlation with model performance~\cite{svmtl}.

\para{Poisoning Attacks.}
Poisoning attack pollutes training data to alter a model's behavior.
Different from backdoor attack, it does not rely on any trigger,
and manipulates the model's behavior on a set of clean samples. Defenses against
poisoning attack mostly focus on sanitizing the training set and removing
poisoned samples~\cite{cao2018efficient, jagielski2018manipulating,
poisondefenseimc, medpos, certpos, sensorpos}. The insight is to find samples
that would alter model's performance significantly~\cite{cao2018efficient}.
This insight is less effective against backdoor
attack~\cite{chen2017targeted}, as injected samples do not affect the model's
performance on clean samples. It is also impractical under our attack model,
as the defender does not have access to the poisoned training set
(used by
the Teacher).

\para{Other Adversarial Attacks against DNNs.}  Numerous (non-backdoor)
adversarial attacks have been proposed against general DNNs, often crafting
imperceptible modifications to images to cause misclassification. These can
be applied to DNNs during inference~\cite{carliniattack, attackscale,
  papernotattack, delvingblackbox, ourtlpaper}.  A number of defenses have
been proposed~\cite{papernotdistillation, mitdefense, logitpairing, magnet,
  featuresqueezing}, yet many have shown to be less effective against an
adaptive adversary~\cite{distallationbroken, ensemblebroken, magnetbroken,
  obfuscatedicml}.  Recent works tried to craft universal
perturbations that trigger misclassification for multiple
images in an uninfected DNN~\cite{brown2017adversarial, moosavi2017universal}.

%% file: conclusion.tex
\section{Conclusion}
\label{sec:conclusion}

In this paper, we identify a new, more powerful variant of the backdoor
attack against deep neural networks. Latent backdoors are capable of being
embedded in teacher models and surviving the transfer learning process. As a
result, they are nearly impossible to identify in teacher models, and only
``activated'' once the model is customized to recognize the target label the
attack was designed for, {\em e.g.} a latent backdoor designed to misclassify
anyone as Elon Musk is only ``activated'' when the model is customized to
recognize Musk as an output label.

We demonstrate the effectiveness and practicality of latent backdoors through
extensive experiments and real-world tests. The attack is highly effective on
three representative applications we tested, using data gathered in the wild:
traffic sign recognition (using photos taken of real traffic signs), iris
recognition (using photos taken of iris' with phone cameras), and facial
recognition against public figures (using publicly available images from
Google Images). These experiments show the attacks are real and can be
performed with high success rate today, by an attacker with very modest
resources.  Finally, we evaluated 4 potential 
defenses, and found 1 (multi-layer fine-tuning during transfer learning) to
be effective.

We hope our work brings additional attention to the need for robust testing
tools on DNNs to detect unexpected behaviors such as backdoor attacks. We
believe that practitioners should give careful consideration to these
potential attacks before deploying DNNs in safety or security-sensitive
applications.


%% file: appendix.tex
\appendix
\section*{Appendix}
\label{sec:appendix}
\para{Model Architecture.} Table~\ref{tab:mnist_cnn}, ~\ref{tab:gtsrb_cnn}, and
~\ref{tab:pubfig_cnn} list the detailed architecture of the Teacher
model for the four applications considered by our evaluation in \S\ref{sec:eval}.
These Teacher models span from small (\mnist{}), medium (\traffic) to large models
(\face{} and \iris). We also list the index of every layer in each model.
Note that the index of
pooling layer is counted as its previous layer, as defined conventionally.

\begin{table}[h]
\caption{Teacher model architecture for \mnist. FC stands for fully-connected layer.
Pooling layer's index is counted as its previous layer.}
\label{tab:mnist_cnn}
\vspace{-0.1in}
\centering
\resizebox{1\columnwidth}{!}{
\begin{tabular}{@{}cccccc@{}}
\toprule

Layer Index & Layer Type & \# of Channels & Filter Size & Stride & Activation \\
\midrule

1 & Conv & 16 & 5$\times$5 & 1 & ReLU \\
1 & MaxPool & 16 & 2$\times$2 & 2 & - \\
2 & Conv & 32 & 5$\times$5 & 1 & ReLU \\
2 & MaxPool & 32 & 2$\times$2 & 2 & - \\
3 & FC & 512 & - & - & ReLU \\
4 & FC & 5 & - & - & Softmax \\
\bottomrule

\end{tabular}
}
\end{table}

\vspace{-0.1in}

\begin{table}[h]
  \caption{Teacher model architecture for \traffic.}
  \vspace{-0.1in}
\label{tab:gtsrb_cnn}
\centering
\resizebox{1\columnwidth}{!}{
\begin{tabular}{@{}cccccc@{}}
\toprule

Layer Index & Layer Type & \# of Channels & Filter Size & Stride & Activation \\
\midrule

1 & Conv & 32 & 3$\times$3 & 1 & ReLU \\
2 & Conv & 32 & 3$\times$3 & 1 & ReLU \\
2 & MaxPool & 32 & 2$\times$2 & 2 & - \\
3 & Conv & 64 & 3$\times$3 & 1 & ReLU \\
4 & Conv & 64 & 3$\times$3 & 1 & ReLU \\
4 & MaxPool & 64 & 2$\times$2 & 2 & - \\
5 & Conv & 128 & 3$\times$3 & 1 & ReLU \\
6 & Conv & 128 & 3$\times$3 & 1 & ReLU \\
6 & MaxPool & 128 & 2$\times$2 & 2 & - \\
7 & FC & 512 & - & - & ReLU \\
8 & FC & 43 & - & - & Softmax \\
\bottomrule

\end{tabular}
}
\end{table}

\vspace{-0.1in}

\begin{table}[h]
  \caption{Teacher model architecture for \face{} and \iris{}.}
\vspace{-0.1in}
\label{tab:pubfig_cnn}
\centering
\resizebox{1\columnwidth}{!}{
\begin{tabular}{@{}cccccc@{}}
\toprule

Layer Index & Layer Type & \# of Channels & Filter Size & Stride & Activation \\
\midrule

1 & Conv & 64 & 3$\times$3 & 1 & ReLU \\
2 & Conv & 64 & 3$\times$3 & 1 & ReLU \\
2 & MaxPool & 64 & 2$\times$2 & 2 & - \\
3 & Conv & 128 & 3$\times$3 & 1 & ReLU \\
4 & Conv & 128 & 3$\times$3 & 1 & ReLU \\
4 & MaxPool & 128 & 2$\times$2 & 2 & - \\
5 & Conv & 256 & 3$\times$3 & 1 & ReLU \\
6 & Conv & 256 & 3$\times$3 & 1 & ReLU \\
7 & Conv & 256 & 3$\times$3 & 1 & ReLU \\
7 & MaxPool & 256 & 2$\times$2 & 2 & - \\
8 & Conv & 512 & 3$\times$3 & 1 & ReLU \\
9 & Conv & 512 & 3$\times$3 & 1 & ReLU \\
10 & Conv & 512 & 3$\times$3 & 1 & ReLU \\
10 & MaxPool & 512 & 2$\times$2 & 2 & - \\
11 & Conv & 512 & 3$\times$3 & 1 & ReLU \\
12 & Conv & 512 & 3$\times$3 & 1 & ReLU \\
13 & Conv & 512 & 3$\times$3 & 1 & ReLU \\
13 & MaxPool & 512 & 2$\times$2 & 2 & - \\
14 & FC & 4096 & - & - & ReLU \\
15 & FC & 4096 & - & - & ReLU \\
16 & FC & 2622 & - & - & Softmax \\
\bottomrule

\end{tabular}
}
\end{table}

\para{Target-dependent Trigger Generation.} Figure~\ref{fig:trigger_opt} shows
samples of backdoor triggers generated by our attacks as discussed in \S\ref{sec:eval}.
The trigger mask is chosen to be a square-shaped pattern
located at the bottom right of each input image. The trigger generation
process maximizes the trigger effectiveness against $y_t$ by minimizing
the difference between poisoned non-target samples and
clean target samples described by eq. (\ref{eq:trigger_rev}). These generated
triggers are  used to inject latent backdoor into the Teacher
model. They are also used to  launch
misclassification attacks after any Student model is trained from the
infected Teacher model.

\begin{figure}[h]
  \centering
  \subfigure[\mnist]{
  	\includegraphics[width=0.42\textwidth]{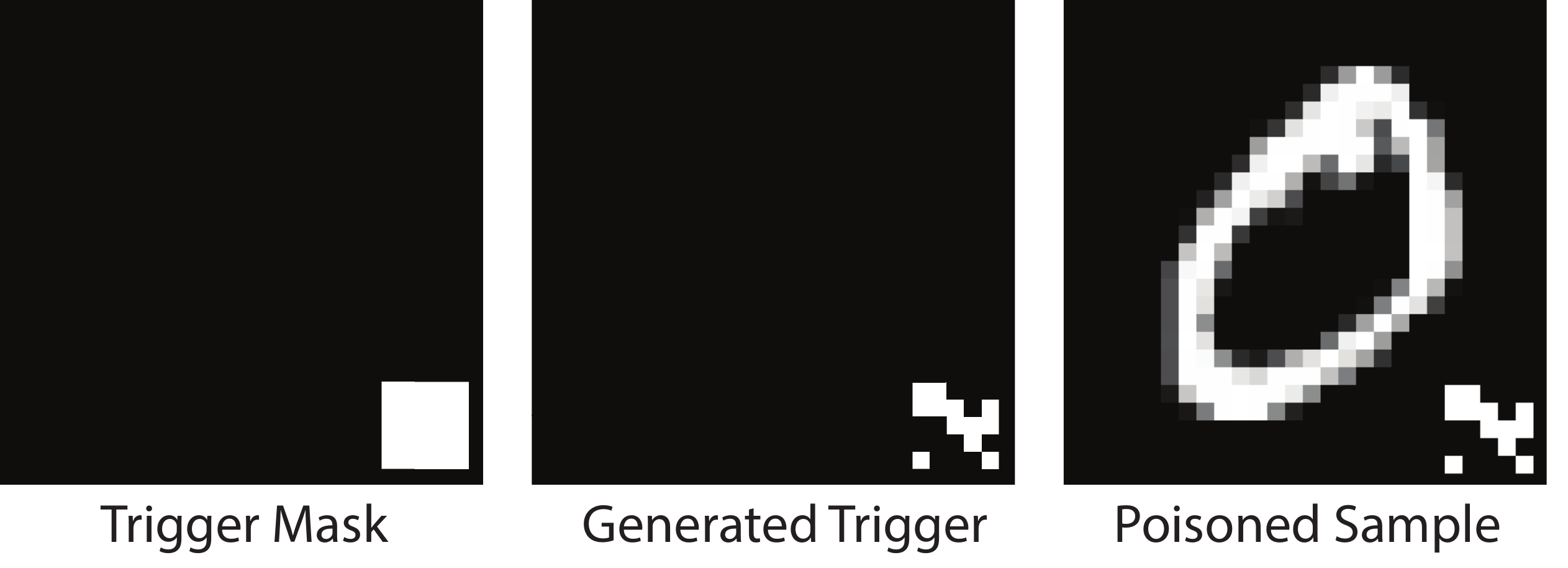}
  }
  \hfill
  \subfigure[\traffic]{
  	\includegraphics[width=.42\textwidth]{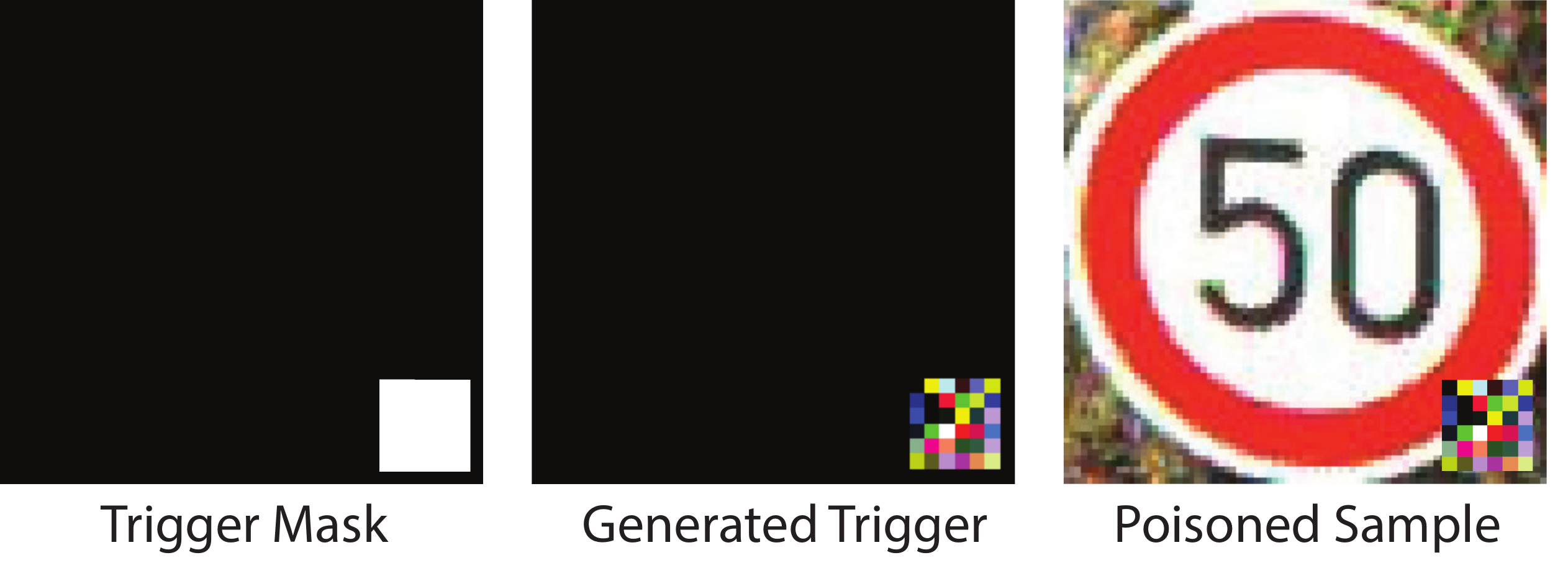}
  }
  \\
  \subfigure[\face]{
  	\includegraphics[width=.42\textwidth]{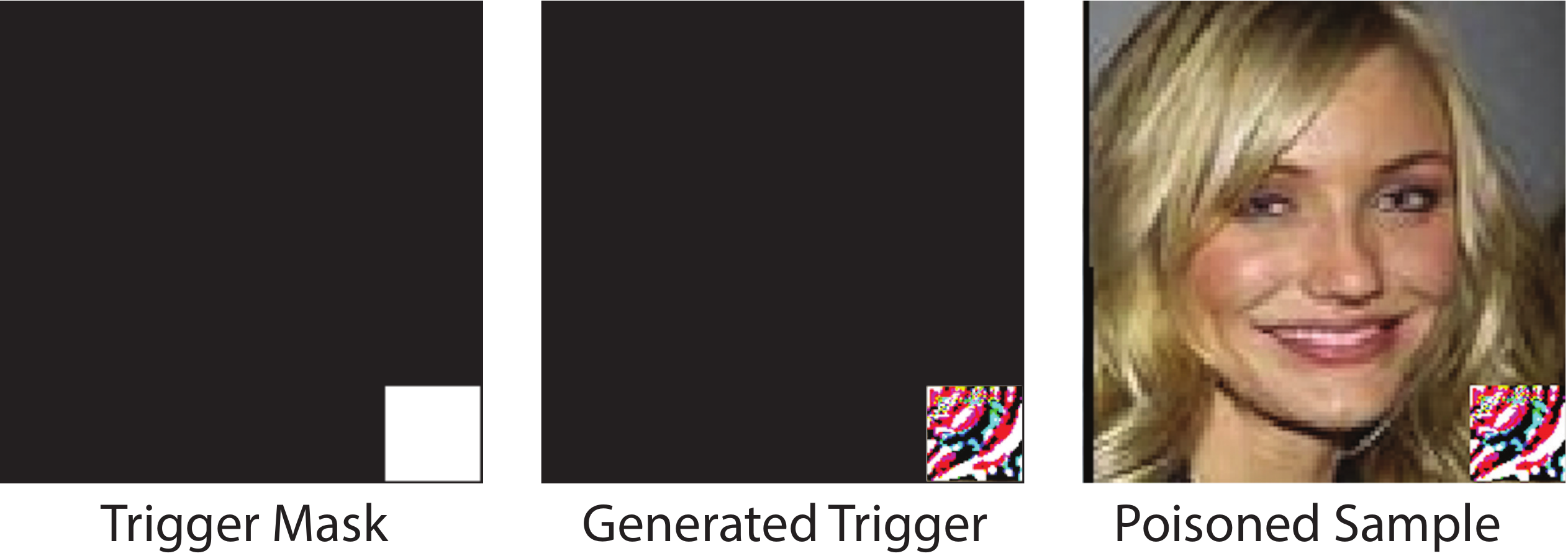}
  }
  \hfill
  \subfigure[\iris]{
  	\includegraphics[width=.42\textwidth]{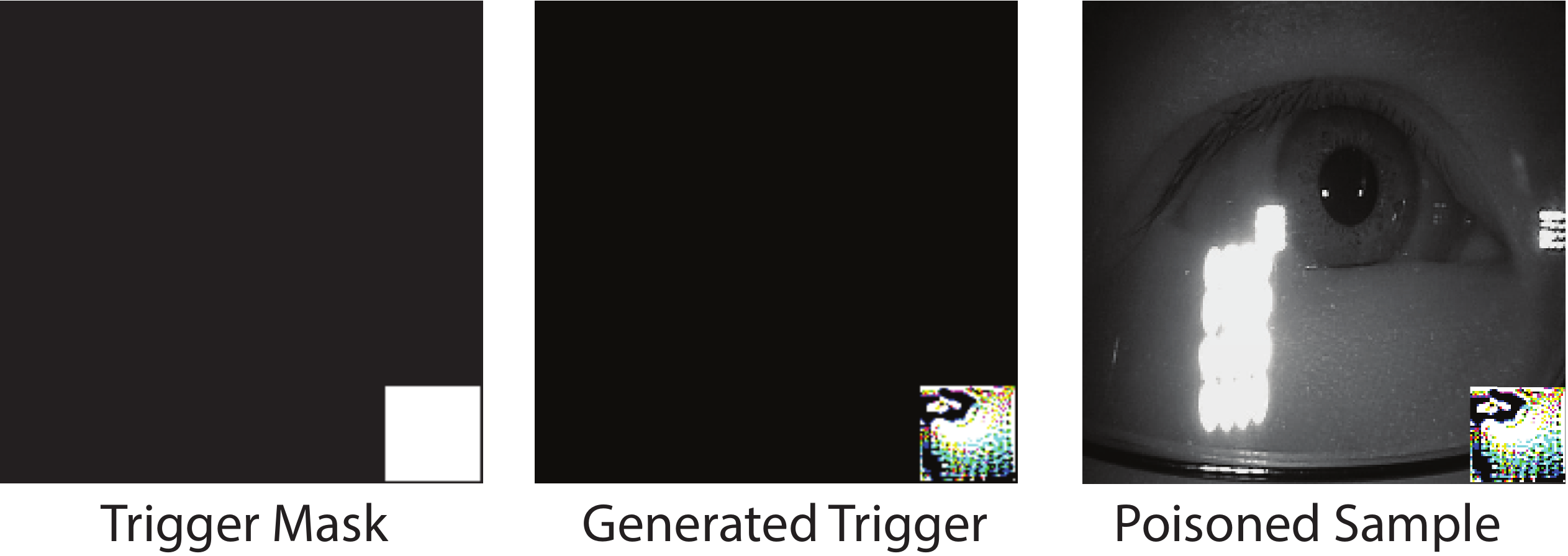}
      }
      \vspace{-0.1in}
  \caption{Samples of triggers produced by our attack and the
    corresponding poisoned images. }
  \label{fig:trigger_opt}
\end{figure}